\documentclass[11pt]{article}
\usepackage{amsthm,amsmath} 
\usepackage[colorlinks,linkcolor=red,anchorcolor=blue,citecolor=blue]{hyperref} 

\usepackage{fullpage}
\usepackage{graphicx}
\usepackage{xspace}
\usepackage[colorinlistoftodos, textsize=tiny, shadow,color=blue!30!white
					,disable
]{todonotes} 
\usepackage{algorithm}
\usepackage{algorithmic}
\RequirePackage{times}
\RequirePackage{natbib}
\usepackage{authblk}

\newcommand{\Prob}[1]{\mathbb P\left(#1\right)}
\newcommand{\E}{\mathbb E}

\newcommand{\R}{\mathbb{R}}
\DeclareMathOperator{\poly}{poly}
\newcommand{\algoname}{\textsc{Confident MC-LSPI}}

\newcommand{\rolloutname}{\textsc{ConfidentRollout}}

\newcommand{\init}{\text{init}}

\numberwithin{remark}{section}
\newtheorem{lemma}{Lemma}
\numberwithin{lemma}{section}
\usepackage{smile}

\begin{document}

\title{ \bf Efficient Local Planning with Linear Function Approximation}

\author[1]{Dong Yin}
\author[1]{Botao Hao}
\author[1]{Yasin Abbasi-Yadkori}
\author[1]{Nevena Lazi\'{c}}
\author[1,2]{Csaba Szepesv\'{a}ri}
\affil[1]{DeepMind \thanks{Emails: \{dongyin, bhao, yadkori, nevena, szepi\}@google.com}}
\affil[2]{University of Alberta}

\date{\today}

\maketitle
\begin{abstract}
We study query and computationally efficient planning algorithms for discounted Markov decision processes (MDPs) with linear function approximation and a simulator.
We assume that the agent has \emph{local access} to the simulator, meaning that the simulator can be queried only for states that have been encountered in previous simulation steps.
This is a more practical setting than the so-called \emph{random-access} (or, generative) setting, where the agent has a complete description of the state space and features and is allowed to query the simulator at any state of its choice.
We propose two new algorithms for this setting, which we call \emph{confident Monte Carlo least-squares policy iteration} (\textsc{Confident MC-LSPI}), and \emph{confident Monte Carlo Politex} (\textsc{Confident MC-Politex}), respectively. 
The main novelty in our algorithms is that it gradually builds a set of state-action pairs (``core set'') with which it can control the extrapolation errors. 
Under the assumption that the action-value functions of all policies 
are linearly realizable with given features, 
we show that our algorithm has polynomial query and computational cost
in the dimension of the features, the effective planning horizon and the targeted sub-optimality, while the cost remains independent of the size of the state space. Our result strengthens previous works by broadening their scope, either by weakening the assumption made on the power of the function approximator, or by weakening the requirement on the simulator and removing the need for being given an appropriate core set of states.
An interesting technical contribution of our work is the introduction of a novel proof technique that makes use of a \emph{virtual policy iteration} algorithm.
We use this method to leverage existing results on $\ell_\infty$-bounded approximate policy iteration to show that our algorithm can learn the optimal policy for the given initial state even only with local access to the simulator. We believe that this technique can be extended to broader settings beyond this work.
\end{abstract}

\section{Introduction}\label{sec:intro}

Efficient planning lies at the heart of modern reinforcement learning (RL). In the simulation-based RL, the agent has access to a simulator 
which it uses to query a state-action pair to obtain the reward of the queried pair and the next state. 
When planning with large state spaces in the presence of features, the agent can also compute the feature vector associated with a state or a state-action pair.
Planning efficiency is measured in two ways: using \emph{query cost}, the number of calls to the simulator, and using \emph{computation cost}, the total number of logical and arithmetic operations that the agent uses. 
In Markov decision processes (MDPs) with a large state space, we call a planning algorithm \emph{query-efficient} (\emph{computationally-efficient}) if its query (respectively, computational) cost is independent of the size of the state space and polynomial in other parameters of the problem such as the dimension of the feature space, the effective planning horizon, the number of actions and the targeted sub-optimality.

Prior works on planning in MDPs often assume that the agent has access to a \emph{generative model} which allows the agent to query the simulator with any \emph{arbitrary} state-action pair~\citep{kakade2003sample, sidford2018near, yang2019sample, lattimore2020learning}. In what follows, we will call this the \emph{random access} model. The random access model is often difficult to support.
To illustrate this, consider a problem where the goal is to  move the joints of a robot arm so that it moves objects around.
The simulation state in this scenario is then completely described by the position, orientation and associated velocities of the various rigid objects involved.
To access a state, a planner can then try to choose some values for each of the variables involved.
Unfortunately, given only box constraints on the variable values (as is typically the case), a generic planner will often choose value combinations that are invalid based on physics, for example with objects penetrating each other in space. 
This problem is not specific to robotic applications but also arises in MDPs corresponding to combinatorial search, just to mention a second example. 

To address this challenge, we replace the random access model with a \emph{local access} model, where the only states at which the agent can query the simulator are the initial states provided to the agent, or states returned in response to previously issued queries.
This access model can be implemented with any simulator that
supports resetting its internal state to a previously stored such state. This type of checkpointing is widely supported, and if a simulator does not support it, there are general techniques that can be applied to achieve this functionality.
As such, this access model significantly expands the scope of planners.

\begin{definition}[Local access to the simulator]\label{def:local_access}
We say the agent has local access to the simulator if the agent is allowed to query the simulator with a state that the agent has previously seen paired with an arbitrary action.
\end{definition}

Our work relies on linear function approximation. Very recently, \cite{weisz2021exponential} showed that linear realizability assumption of the optimal state-action value function ($Q^*$-realizability) alone is not sufficient to develop a query-efficient planner. In this paper, we assume linear realizability of all policies ($Q_{\pi}$-realizability). We discuss several drawbacks of previous works \citep{lattimore2020learning, du2020good} under the same realizability assumption. First, these works require the knowledge of the features of all state-action pairs; otherwise, the agent has to spend $\cO(|\cS||\cA|)$ query cost to extract the features of all possible state-action pairs, where $|\cS|$ and $|\cA|$ are the sizes of the state space and action space, respectively. Second, these algorithms require the computation of either an approximation of the global optimal design \citep{lattimore2020learning} or a barycentric spanner \citep{du2020good} of the matrix of all features. Although there exists algorithms to approximate the optimal design~\citep{todd2016minimum} or barycentric spanner~\citep{awerbuch2008online}, the computational complexities for these algorithms are polynomial in the total number of all possible feature vectors, i.e., $|\cS||\cA|$, which is impractical for large MDPs.

We summarize our contributions as follows:
\begin{itemize}
    \item With local access to the simulator, we propose two policy optimization algorithms---\emph{confident Monte Carlo least-squares policy iteration} (\algoname), and its regularized (see e.g.~\citet{even2009online,abbasi2019politex}) version \emph{confident Monte Carlo Politex} (\textsc{Confident MC-Politex}). Both of our algorithms maintain a \emph{core set} of state-action pairs and run Monte Carlo rollouts from these pairs using the simulator. The algorithms then use the rollout results to estimate the Q-function values and then apply policy improvement. During each rollout procedure, whenever the algorithm observes a state-action pair that it is less confident about (with large uncertainty), the algorithm adds this pair to the core set and restarts. Compared to several prior works that use additive bonus~\citep{jin2020provably,cai2020provably}, our algorithm design demonstrates that in the local access setting, core-set-based exploration is an effective approach.
\vspace{-1mm}
    \item Under the $Q_{\pi}$-realizability assumption, 
    we prove that both \algoname~and \textsc{Confident MC-Politex} can learn a $\kappa$-optimal policy with query cost of $\poly(d, \frac{1}{1-\gamma}, \frac{1}{\kappa}, \log(\frac{1}{\delta}), \log(b))$ and computational costs of $\poly(d, \frac{1}{1-\gamma}, \frac{1}{\kappa}, |\cA|, \log(\frac{1}{\delta}), \log(b))$, where $d$ is the dimension of the feature of state-action pairs, $\gamma$ is the discount factor, $\delta$ is the error probability, and $b$ is the bound on the $\ell_2$ norm of the linear coefficients for the Q-functions. In the presence of a model misspecification error $\epsilon$, we show that \algoname~achieves a final sub-optimality of $\tilde{\cO}(\frac{\epsilon\sqrt{d}}{(1-\gamma)^2})$, whereas \textsc{Confident MC-Politex} can improve the sub-optimality to $\tilde{\cO}(\frac{\epsilon\sqrt{d}}{1-\gamma})$ with a higher query cost. 
   \vspace{-1mm}
    \item We develop a novel proof technique that makes use of a \emph{virtual policy iteration} algorithm. We use this method to leverage existing results on approximate policy iteration which assumes that in each iteration, the approximation of the Q-function has a bounded $\ell_\infty$ error~\citep{munos2003error,farahmand2010error} (see Section \ref{sec:guarantee} for details).
\end{itemize}

\section{Related work}\label{sec:related_work}
Simulators or generative models have been considered in early studies of reinforcement learning~\citep{kearns1999finite, kakade2003sample}. Recently, it has been shown empirically that in the local access setting, core-set-based exploration has strong performance in hard-exploration problems~\citep{ecoffet2019go}. In this section, we mostly focus on related theoretical works. We distinguish among \emph{random access}, \emph{local access}, and \emph{online access}.
\begin{itemize}
    \item Random access means that the agent is given a list of all possible state action pairs and can query any of them to get the reward and a sample of the next state.
    \item Local access means that the agent can access previously encountered states, which can be implemented with checkpointing. The local access model that we consider in this paper is a more practical version of planning with a simulator.
    \item Online access means that the simulation state can only be reset to the initial state (or distribution) or moved to a next random state given an action. The online access setting is more restrictive compared to local access, since the agent can only follow the MDP dynamics during the learning process.
\end{itemize}
We also distinguish between \emph{offline} and \emph{online} planning. In the offline planning problem, the agent only has access to the simulator during the training phase, and once the training is finished, the agent outputs a policy and executes the policy in the environment without access to a simulator. This is the setting that we consider in this paper. On the other hand, in the online planning problem, the agent can use the simulator during both the training and inference phases, meaning that the agent can use the simulator to choose the action when executing the policy.
Usually, online RL algorithms with sublinear regret can be converted to an offline planning algorithm under the online access model with standard online-to-batch conversion~\citep{cesa2004generalization}.
While most of the prior works that we discuss in this section are for the offline planning problem, the \textsc{TensorPlan} algorithm~\citep{weisz2021query} considers online planning. 

In terms of notation, some works considers finite-horizon MDPs, in which case we use $H$ to denote the episode length (similar to the effective planning horizon $(1-\gamma)^{-1}$ in infinite-horizon discounted MDPs). 
Our discussion mainly focuses on the results with linear function approximation. We summarize some of the recent advances on efficient planning in large MDPs in Table \ref{table:comparsion}.

\begin{table}[ht]
\centering
\vspace{-3mm}
\caption{Recent advances on RL algorithms with linear function approximation under different assumptions. Positive results mean query cost depends only polynomially on the relative parameter while negative results refer an exponential lower bound on the query complexity. CE stands for computational efficiency and ``no'' for CE means no computational efficient algorithm is provided. \\ $\dag$: The algorithms in these works are not query or computationally efficient unless the agent is provided with an approximate optimal design~\citep{lattimore2020learning} or barycentric spanner~\citep{du2020good} or ``core states''~\citep{shariff2020efficient} for free. \\ $\ddag$:~\citet{weisz2021query} consider the online planning problem whereas other works in this table consider (or can be converted to) the offline planning problem.
}\label{table:comparsion}
\scalebox{0.88}{
\begin{tabular}{ cccc } 
 \toprule
 Positive Results & Assumption & CE & Access Model \\ 
 \midrule
 \cite{yang2019sample}&linear MDP & yes & random access\\ 
 \hline
 \cite{lattimore2020learning, du2020good} & $Q_{\pi}$-realizability & no $\dagger$  & random access \\ 
 \hline
 \citet{shariff2020efficient} &  $V^*$-realizability & no $\dagger$  & random access \\ 
 \hline
 \textbf{This work}&$Q_{\pi}$\textbf{-realizability} & \textbf{yes}  & \textbf{local access}\\
 \hline
\cite{weisz2021query} $\ddag$& $V^*$-realizability, $\cO(1)$ actions & no  & local access\\ 
 \hline
 \cite{li2021sample}& $Q^{*}$-realizability, constant gap & yes  & local access\\ 
\hline
\cite{jiang2017contextual} &  low Bellman rank & no  & online access \\ 
\hline
\cite{zanette2020learning} & low inherent Bellman error & no  & online access \\ 
\hline
    \cite{du2021bilinear} & bilinear class & no & online access \\
  \hline
\cite{lazic2021improved,wei2021learning} & $Q_\pi$-realizability, feature excitation & yes  & online access \\ 
\hline
\cite{jin2020provably,agarwal2020pc} & linear MDP & yes  & online access \\ 
 \hline
  \cite{zhou2020provably,cai2020provably} & linear mixture MDP & ? & online access \\
  \hline
 \midrule
  Negative Results & Assumption & CE & Access Model \\ 
 \midrule
 \cite{du2020good} & $Q_\pi$-realizability, $\epsilon=\Omega(\sqrt{H/d})$ & N/A & random access
 \\
 \hline
 \cite{weisz2021exponential}&$Q^*$-realizability, $\exp(d)$ actions & N/A & random access\\
 \hline
  \cite{wang2021exponential}&$Q^*$-realizability, constant gap & N/A & online access \\ 
 \bottomrule
\end{tabular}}
\vspace{-3mm}
\end{table}

\paragraph{Random access}
Theoretical guarantees for the random access model have been obtained for the tabular setting~\citep{sidford2018near,agarwal2020model,li2020breaking,azar2013minimax}.
As for linear function approximation, different assumptions have been made for theoretical analysis.
Under the linear MDP assumption,
\cite{yang2019sample} derived an optimal $\cO(d\kappa^{-2}(1-\gamma)^{-3})$ query complexity bound by a variance-reduction Q-learning type algorithm.
Under the $Q_\pi$-realizability of all determinstic policies (a strictly weaker assumption than linear MDP~\citep{zanette2020learning}), \citet{du2020good} showed a negative result for the settings with model misspecification error $\epsilon=\Omega(\sqrt{H/d})$ (see also~\citet{van2019comments,lattimore2020learning}). When $\epsilon = o((1-\gamma)^2/\sqrt{d})$, assuming the access to the full feature matrix, \citet{lattimore2020learning} proposed algorithms with polynomial query costs, and \citet{du2020good} proposed similar algorithm for the exact $Q_\pi$ realizability setting. Since these works need to find a globally optimal design or barycentric spanner, their computational costs depend polynomially on the size of the state space. Under the $V^*$-realizability assumption (i.e., the optimal value function is linear in some feature map), \citet{shariff2020efficient} proposed a planning algorithm assuming the availability of a set of core states but obtaining such core states can still be computationally inefficient. \citet{zanette2019limiting} proposed an algorithm that uses a similar concept named anchor points but only provided a greedy heuristic to generate these points.
A notable negative result is established in \cite{weisz2021exponential} that shows that with only $Q^*$-realizability, any agent requires $\min(\exp(\Omega(d)), \exp(\Omega(H)))$ queries to learn an optimal policy.

\paragraph{Local access}
Many prior studies have used simulators in tree-search style algorithms~\citep{kearns2002sparse,munos2014bandits}.
Under this setting, for the online planning problem, recently \cite{weisz2021query} established an $\cO((dH/\kappa)^{|\cA|})$ query cost bound to learn an $\kappa$-optimal policy by the \textsc{TensorPlan} algorithm assuming the $V^*$-realizability. Whenever the action set is small, \textsc{TensorPlan} is query efficient, but its computational efficiency is left as an open problem. Under $Q^*$-realizability and constant sub-optimality gap, for the offline planning problem, \cite{li2021sample} proposed an algorithm with $\poly(d, H, \kappa^{-1}, \Delta_{\text{gap}}^{-1})$ query and computational costs.

\paragraph{Online access}
As mentioned, many online RL algorithms can be converted to a policy optimization algorithm under the online access model using online-to-batch conversion. There is a large body of literature on online RL with linear function approximation and here we discuss a non-exhaustive list of prior works. Under the $Q^*$-realizability assumption, assuming that the probability transition of the MDP is deterministic, \citet{wen2013efficient} proposed a sample and computationally efficient algorithm via the eluder dimension~\citep{russo2013eluder}.
Assuming the MDP has low Bellman rank, \citet{jiang2017contextual} proposed an algorithm that is sample efficient but computationally inefficient, and similar issues arise in~\citet{zanette2020learning} under the low inherent Bellman error assumption.
\citet{du2021bilinear} proposed a more general MDP class named \emph{bilinear class} and provided a sample efficient algorithm, but the computational efficiency is unclear.

Under $Q_\pi$-realizability, several algorithms, such as \textsc{Politex}~\citep{abbasi2019politex, lazic2021improved}, AAPI~\citep{hao2021adaptive}, and MDP-EXP2~\citep{wei2021learning} achieved sublinear regret in the infinite horizon average reward setting and are also computationally efficient. However, the corresponding analysis avoids the exploration issue by imposing a \emph{feature excitation} assumption which may not be satisfied in many problems. Under the linear MDP assumption, \cite{jin2020provably} established a $\cO(\sqrt{d^3H^3T})$ regret bound for an optimistic least-square value iteration algorithm. \cite{agarwal2020pc} derived a $\poly(d, H, \kappa^{-1})$ sample cost bound for the policy cover-policy gradient algorithm, which can also be applied in the state aggregation setting; the algorithm and sample cost were subsequently improved in \cite{zanette2021cautiously}.
Under the linear mixture MDP assumption \citep{yang2019reinforcement,zhou2020provably}, \cite{cai2020provably} proved an $\cO(\sqrt{d^3H^3T})$ regret bound for an optimistic least square policy iteration (LSPI) type algorithm.
A notable negative result for the online RL setting by~\citet{wang2021exponential} shows that an exponentially large number of samples are needed if we only assume $Q^*$-realizability and constant sub-optimality gap.
Other related works include~\citet{ayoub2020model,jin2021bellman,du2019provably,wang2019optimism}, and references therein.

\vspace{-2mm}
\section{Preliminaries}\label{sec:prelim}
\vspace{-1mm}
We use $\Delta_{\cS}$ to denote the set of probability distributions defined on the set $\cS$.
Consider an infinite-horizon discounted MDP that is specified by a tuple $(\cS, \cA, r, P, \rho, \gamma)$, where $\cS$ is the state space, $\cA$ is the finite action space, $r : \cS \times \cA \rightarrow [0, 1]$ is the reward function, $P : \cS \times \cA \rightarrow \Delta_\cS$ is the probability transition kernel, $\rho\in\cS$ is the initial state, and $\gamma \in (0, 1)$ is the discount factor. For simplicity, in the main sections of this paper, we assume that the initial state $\rho$ is deterministic and known to the agent. Our algorithm can also be extended to the setting where the initial state is random and the agent is allowed to sample from the initial state distribution. We discuss this extension in Appendix~\ref{apx:random_initial}. Throughout this paper, we write $[N] := \{1,2, \ldots, N\}$ for any positive integer $N$ and use $\log(\cdot)$ to denote natural logarithm.

A policy $\pi : \cS \rightarrow \Delta_\cA$ is a mapping from a state to a distribution over actions. We only consider stationary policies, i.e., they do not change according to the time step.
The value function $V_\pi(s)$ of a policy is the expected return when we start running the policy $\pi$ from state $s$, i.e.,
\[
V_\pi(s) = \E_{a_t \sim \pi(\cdot | s_t), s_{t+1}\sim P(\cdot | s_t, a_t)} \left[ \sum_{t=0}^\infty \gamma^t r(s_t, a_t) \mid s_0=s \right],
\]
and the state-action value function $Q_\pi(s, a)$, also known as the Q-function, is the expected return following policy $\pi$ conditioned on $s_0=s, a_0=a$, i.e.,
\[
Q_\pi(s, a) = \E_{s_{t+1}\sim P(\cdot | s_t, a_t), a_{t+1} \sim \pi(\cdot | s_{t+1})} \left[ \sum_{t=0}^\infty \gamma^t r(s_t, a_t) \mid s_0=s, a_0=a \right].
\]
We assume that the agent interacts with a simulator using the local access protocol defined in Definition~\ref{def:local_access}, i.e, for any state $s$ that the agent has visited and any action $a\in\cA$, the agent can query the simulator and obtain a sample $s'\sim P(\cdot | s, a)$ and the reward $r(s, a)$.

Our general goal is to find a policy that maximizes the expected return starting from the initial state $\rho$, i.e., $\max_{\pi}V_{\pi}(\rho)$. We let $\pi^*$ be the optimal policy, $V^*(\cdot) := V_{\pi^*}(\cdot)$, and $Q^*(\cdot, \cdot) := Q_{\pi^*}(\cdot, \cdot)$. We also aim to learn a good policy efficiently, i.e., the query and computational costs should not depend on the size of the state space $\cS$, which can be large in many problems.

\paragraph{Linear function approximation}
Let $\phi: \cS \times \cA \to \RR^d$ be a feature map which assigns to each state-action pair a $d$-dimensional feature vector. For any $(s, a)\in \cS \times \cA$, the agent can obtain $\phi(s, a)$ with a computational cost of $\poly(d)$. Here, we emphasize that the computation of the feature vectors does not lead to a query cost. Without loss of generality, we impose the following bounded features assumption.
\vspace{-1mm}
\begin{assumption}[Bounded features] \label{asm:bounded_norm}
We assume that $\|\phi(s, a)\|_2 \le 1$ for all $(s, a)\in\cS \times \cA$.
\vspace{-2mm}
\end{assumption}

We consider the following two different assumptions on the linear realizability of the Q-functions:
\begin{assumption}[$Q_\pi$-realizability]\label{asm:linear_realizability}
There exists $b>0$ such that for every policy $\pi$, there exists a weight vector $w_\pi\in\R^d$, $\|w_\pi\|_2\le b$, that ensures $Q_\pi(s, a) = w_\pi^\top \phi(s, a) $ for all $(s, a)\in\cS \times \cA$.
\vspace{-2mm}
\end{assumption}

\begin{assumption}[Approximate $Q_\pi$-realizability]\label{asm:linear_realizability_prox}
There exists $b>0$ and model misspecification error $\epsilon>0$ such that for every policy $\pi$, there exists a weight vector $w_\pi\in\R^d$, $\|w_\pi\|_2\le b$, that ensures $|Q_\pi(s, a) - w_\pi^\top \phi(s, a)| \le \epsilon$ for all $(s, a)\in\cS \times \cA$.
\end{assumption}
\vspace{-3mm}
\vspace{-2mm}
\section{Algorithm}\label{sec:algo}
\vspace{-1mm}
We first introduce some basic concepts used in our algorithms.

\paragraph*{Core set} We use a concept called core set. A core set $\cC$ is a set of tuples $z = (s, a, \phi(s, a), q) \in \cS \times \cA \times \R^d \times (\R\cup\{\none\})$. The first three elements in the tuple denote a state, an action, and the feature vector corresponding to the state-action pair, respectively. The last element $q\in\R$ in the tuple denotes an \emph{estimate} of $Q_\pi(s, a)$ for a policy $\pi$. During the algorithm, we may not always have such an estimate, in which case we write $q=\none$. For a tuple $z$, we use $z_s$, $z_a$, $z_\phi$, and $z_q$ to denote the $s$, $a$, $\phi$, and $q$ coordinates of $z$, respectively. We note that in prior works, the core set usually consists of the state-action pairs and their features~\citep{lattimore2020learning,du2020good,shariff2020efficient}; whereas in this paper, for the convenience of notation, we also have the target values (Q-function estimates) in the core set elements. We denote by $\Phi_{\cC} \in \R^{|\cC|\times d}$ the \emph{feature matrix} of all the elements in $\cC$, i.e., each row of $\Phi_{\cC}$ is the feature vector of an element in $\cC$. Similarly, we define $q_{\cC}\in\R^{|\cC|}$ as the vector for the $Q_\pi$ estimate of all the tuples in $\cC$.

\paragraph*{Good set} It is also useful to introduce a notion of \emph{good set}.
\begin{definition}\label{def:confident_set}
Given $\lambda,\tau > 0$, and feature matrix $\Phi_{\cC} \in \R^{|\cC|\times d}$, the good set $\cH\subset \R^d$ is defined as
\[
\cH := \{\phi \in \R^d :  \phi^\top (\Phi_{\cC}^\top \Phi_{\cC} + \lambda I)^{-1} \phi \le \tau \}.
\]
\end{definition}
Intuitively, the good set is a set of vectors that are well-covered by the rows of $\Phi_{\cC}$; in other words, these vectors are not closely aligned with the eigenvectors associated with the small eigenvalues of the covariance matrix of all the features in the core set.

As an overview, our algorithm \algoname~works as follows. First, we initialize the core set using the initial state $\rho$ paired with all actions. Then, the algorithm runs least-squares policy iteration~\citep{munos2003error} to optimize the policy. This means that in each iteration, we estimate the Q-function value for every state-action pair in $\cC$ using Monte Carlo rollout with the simulator, and learn a linear function to approximate the Q-function of the rollout policy, and the next policy is chosen to be greedy with respect to this linear function. Our second algorithm~\textsc{Confident MC-Politex} works similarly, with the only difference being that instead of using the greedy policy iteration update rule, we use the mirror descent update rule with KL regularization between adjacent rollout policies~\citep{even2009online,abbasi2019politex}. Moreover, in both algorithms, whenever we observe a state-action pair whose feature is not in the good set during Monte Carlo rollout, we add the pair to the core set and restart the policy iteration process. We name the rollout subroutine \rolloutname. We discuss details in the following.

\vspace{-2mm}
\subsection{Subroutine: \rolloutname}
\vspace{-1mm}
We first introduce the \rolloutname~subroutine, whose purpose is to estimate $Q_\pi(s_0, a_0)$ for a given state-action pair $(s_0, a_0)$ using Monte Carlo rollouts. During a rollout, for each state $s$ that we encounter and all actions $a\in\cA$, the subroutine checks whether the feature vector $\phi(s, a)$ is in the good set. If not, we know that we have discovered a new feature direction, i.e. a direction which is not well aligned with eigenvectors corresponding to the the largest eigenvalues of the covariance matrix of the core features.
In this case the subroutine terminates and returns the tuple $(s, a, \phi(s,a), \none)$ along with the $\mathsf{uncertain}$ status. If the algorithm does not discover a new direction, it returns an estimate $q$ of the desired value $Q_\pi(s_0, a_0)$ and the $\mathsf{done}$ status. This subroutine is formally presented in Algorithm~\ref{alg:rollout}.

\begin{algorithm}[ht]
\caption{\rolloutname}
\begin{algorithmic}[1]\label{alg:rollout}
\STATE \textbf{Input:} number of rollouts $m$, length of rollout $n$, rollout policy $\pi$, discount $\gamma$, initial state $s_0$, initial action $a_0$, feature matrix $\Phi_{\cC}$, regularization coefficient $\lambda$, threshold $\tau$.
\FOR{$i=1,\ldots, m$}
\STATE $s_{i,0} \leftarrow s_0$, $a_{i, 0} \leftarrow a_0$, query the simulator, obtain reward $r_{i, 0} \leftarrow r(s_{i, 0}, a_{i, 0})$, and next state $s_{i, 1}$.
\FOR{$t=1,\ldots, n$}
\FOR{$a\in\cA$}
\STATE Compute feature $\phi(s_{i, t}, a)$.
\IF{$\phi(s_{i,t}, a)^\top (\Phi_{\cC}^\top \Phi_{\cC} + \lambda I)^{-1} \phi(s_{i,t}, a) > \tau$}
\STATE $\status \leftarrow \uncertain$, $\result\leftarrow (s_{i, t}, a, \phi(s_{i, t}, a), \none)$
\STATE \textbf{return~} $\status, \result$
\ENDIF
\ENDFOR
\STATE Sample $a_{i, t} \sim \pi(\cdot | s_{i, t})$.
\STATE Query the simulator with $s_{i,t}, a_{i, t}$, obtain reward $r_{i,t}\leftarrow r(s_{i,t}, a_{i,t})$, and next state $s_{i, t+1}$.
\ENDFOR
\ENDFOR
\STATE $\status \leftarrow \done$, $\result \leftarrow \frac{1}{m}\sum_{i=1}^m \sum_{t=0}^n \gamma^t r_{i, t}$
\STATE \textbf{return~} $\status, \result$
\end{algorithmic}
\end{algorithm} 

\vspace{-2mm}
\subsection{Policy iteration}
With the subroutine, now we are ready to present our main algorithms. Both of our algorithms maintain a core set $\cC$. We first initialize the core set using the initial state $\rho$ and all  actions $a\in\cA$. More specifically, we check all the feature vectors $\phi(\rho, a), a\in\cA$. If the feature vector is not in the good set of the current core set, we add the tuple $\{(\rho, a, \phi(\rho, a), \none)\}$ to the core set. Then we start the policy iteration process. Both algorithms start with an arbitrary initial policy $\pi_0$ and run $K$ iterations.
Let $\pi_{k-1}$ be the rollout policy in the $k$-th iteration. We try to estimate the state-action values for the state-action pairs in $\cC$ under the current policy $\pi_{k-1}$, i.e., $Q_{\pi_{k-1}}(z_s, z_a)$ for $z\in\cC$, using \rolloutname. In this Q-function estimation procedure, we may encounter two scenarios:
\begin{itemize}
\vspace{-2mm}
    \item[(a)] If the rollout subroutine always returns the $\done$ status with an estimate of the state-action value, once we finish the estimation for all the state-action pairs in $\cC$, we can estimate the Q-function of $\pi_{k-1}$ using least squares with input features $\Phi_{\cC}$ and targets $q_{\cC}$ and regularization coefficient $\lambda$. Let $w_k$ be the solution to the least squares problem, i.e.,
    \begin{align}\label{eq:ridge_sol}
    \vspace{-2mm}
        w_k = (\Phi_{\cC}^\top \Phi_{\cC} + \lambda I)^{-1}\Phi_{\cC}^\top q_{\cC}.
        \vspace{-2mm}
    \end{align}
    Then, for \algoname, we choose the rollout policy of the next iteration, i.e., $\pi_k$, as the greedy policy with respect to the linear function $w_k^\top\phi(s, a)$:
    \begin{align}\label{eq:greedy_policy}
    \pi_{k}(a | s) = \ind \big(a = \arg\max_{a'\in\cA} w_k^\top \phi(s, a')\big).
\end{align}
For \textsc{Confident MC-Politex}, we construct a truncated Q-function $Q_{k-1}:\cS\times\cA\mapsto [0, (1-\gamma)^{-1}]$ using linear function with clipping:
\begin{align}\label{eq:q_prox_politex}
    Q_{k-1}(s, a):=\Pi_{[0, (1-\gamma)^{-1}]}(w_k^\top\phi(s, a)),
\end{align}
where $\Pi_{[a, b]}(x) := \min\{\max\{x, a\}, b\}$. The rollout policy of the next iteration is then
\begin{align}\label{eq:politex_policy}
\vspace{-1mm}
\pi_k(a|s) \propto \exp\big(\alpha \sum_{j=1}^{k-1} Q_{j}(s, a) \big),
\vspace{-1mm}
\end{align}
where $\alpha>0$ is an algorithm parameter.

\vspace{-2mm}
    \item[(b)] It could also happen that the \rolloutname~subroutine returns the $\uncertain$ status. In this case, we add the state-action pair with new feature direction found by the subroutine to the core set and restart the policy iteration process with the latest core set.
\end{itemize}

As a final note, for \algoname, we output the rollout policy of the last iteration $\pi_{K-1}$, whereas for \textsc{Confident MC-Politex}, we output a \emph{mixture policy} $\overline{\pi}_K$, which is a policy chosen from $\{\pi_k\}_{k=0}^{K-1}$ uniformly at random. The reason that this algorithm needs to output a mixture policy is that  \textsc{Politex}~\citep{szepesvari2021politex} uses the regret analysis of expert learning~\citep{cesa2006prediction}, and to obtain a single output policy, we need to use the standard online-to-batch conversion argument~\citep{cesa2004generalization}. Our algorithms are formally presented in Algorithm~\ref{alg:main}. In the next section, we present theoretical guarantees for our algorithms.

\begin{algorithm}[ht]
\caption{\algoname~/~\textsc{Politex}}
\begin{algorithmic}[1]\label{alg:main}
\STATE \textbf{Input:} initial state $\rho$, initial policy $\pi_0$, number of iterations $K$, regularization coefficient $\lambda$, threshold $\tau$, discount $\gamma$, number of rollouts $m$, length of rollout $n$, \textsc{Politex} parameter $\alpha$.
\STATE $\cC \leftarrow \emptyset$ \quad \slash\slash~{\emph{Initialize core set.}}
\FOR{$a\in\cA$}
\IF{$\cC = \emptyset$ or $\phi(\rho, a)^\top(\Phi_{\cC}^\top \Phi_{\cC} + \lambda I)^{-1}\phi(\rho, a) > \tau$}
\STATE $\cC \leftarrow \cC \cup \{(\rho, a, \phi(\rho, a), \none)\}$
\ENDIF
\ENDFOR
\STATE $z_q \leftarrow \none,~\forall z\in\cC$ \quad \slash\slash~{\emph{Policy iteration starts.}} \hfill $(*)$
\FOR{$k=1,\ldots, K$}
\FOR{$z\in\cC$}
\STATE $\status, \result \leftarrow \rolloutname(m, n, \pi_{k-1}, \gamma, z_s, z_a, \Phi_{\cC}, \lambda, \tau)$
\STATE \textbf{if} $\status=\done$, \textbf{then} $z_q \leftarrow \result$;
\textbf{else} $\cC \leftarrow \cC \cup \{\result\}$ and \textbf{goto} line $(*)$
\ENDFOR
\STATE $w_k \leftarrow (\Phi_{\cC}^\top \Phi_{\cC} + \lambda I)^{-1}\Phi_{\cC}^\top q_{\cC}$; $Q_{k-1}(s, a)\leftarrow \Pi_{[0, (1-\gamma)^{-1}]}(w_k^\top\phi(s, a))$ (\textsc{Politex} only)
\STATE $\pi_k(a|s)\leftarrow \begin{cases}
\ind \big(a = \arg\max_{a'\in\cA} w_k^\top \phi(s, a')\big), & \textsc{LSPI}\\
\exp\big(\alpha \sum_{j=1}^{k-1} Q_{j}(s, a) \big) / \sum_{a'\in\cA} \exp\big(\alpha \sum_{j=1}^{k-1} Q_{j}(s, a') \big) & \textsc{Politex}
\end{cases}$
\ENDFOR
\STATE \textbf{return} $w_{K-1}$ for \textsc{LSPI}, or $\overline{\pi}_K\sim\text{Unif}\{\pi_k\}_{k=0}^{K-1}$ for \textsc{Politex}.
\end{algorithmic}
\end{algorithm} 

\vspace{-2mm}
\section{Theoretical guarantees}\label{sec:guarantee}
\vspace{-2mm}
In this section, we present theoretical guarantees for our algorithms. First, we have the following main result for \algoname.
\vspace{-2mm}
\begin{theorem}[Main result for \algoname]\label{thm:main_rand}
If Assumption~\ref{asm:linear_realizability} holds, then for an arbitrarily small $\kappa>0$, by choosing $\tau=1$, $\lambda = \frac{\kappa^2(1-\gamma)^4}{1024 b^2}$, $n=\frac{3}{1-\gamma}\log(\frac{4(1+\log(1+\lambda^{-1})d)}{\kappa(1-\gamma)})$, $K=2+\frac{2}{1-\gamma}\log(\frac{3}{\kappa(1-\gamma)})$, $m=4096 \frac{d(1+\log(1+\lambda^{-1}))}{\kappa^2(1-\gamma)^{6}}\log(\frac{8Kd(1+\log(1+\lambda^{-1}))}{\delta})$, we have with probability at least $1-\delta$, the policy $\pi_{K-1}$ that \algoname~outputs satisfies
\[
\vspace{-2mm}
V^*(\rho) - V_{\pi_{K-1}}(\rho) \le \kappa.
\vspace{-2mm}
\]
Moreover, the query and computational costs for the algorithm are $\poly(d, \frac{1}{1-\gamma}, \frac{1}{\kappa}, \log(\frac{1}{\delta}), \log(b))$ and $\poly(d, \frac{1}{1-\gamma}, \frac{1}{\kappa}, |\cA|, \log(\frac{1}{\delta}), \log(b))$, respectively.

Alternatively, if Assumption~\ref{asm:linear_realizability_prox} holds, then by choosing $\tau=1$, $\lambda = \frac{\epsilon^2 d}{b^2}$, $n=\frac{1}{1-\gamma}\log(\frac{1}{\epsilon(1-\gamma)})$, $K=2 + \frac{1}{1-\gamma} \log\big(\frac{1}{\epsilon\sqrt{d}}\big)$, $m=\frac{1}{\epsilon^2(1-\gamma)^2}\log(\frac{8Kd(1+\log(1+\lambda^{-1}))}{ \delta})$, we have with probability at least $1-\delta$, the policy $\pi_{K-1}$ that \algoname~outputs satisfies
\[
V^*(\rho) - V_{\pi_{K-1}}(\rho) \le \frac{74\epsilon\sqrt{d}}{(1-\gamma)^2}(1+\log(1+b^2\epsilon^{-2}d^{-1})).
\]
Moreover, the query and computational costs for the algorithm are $\poly(d, \frac{1}{1-\gamma}, \frac{1}{\epsilon}, \log(\frac{1}{\delta}), \log(b))$ and $\poly(d, \frac{1}{1-\gamma}, \frac{1}{\epsilon}, |\cA|, \log(\frac{1}{\delta}), \log(b))$, respectively.
\end{theorem}

We prove Theorem~\ref{thm:main_rand} in Appendix~\ref{apx:main_rand_v2}. For \textsc{Confident MC-Politex}, since we output a mixture policy, we prove guarantees for the \emph{expected value of the mixture policy}, i.e., $V_{\overline{\pi}_K}:=\frac{1}{K}\sum_{k=0}^{K-1}V_{\pi_k}$. We have the following result.

\begin{theorem}[Main result for \textsc{Confident MC-Politex}]\label{thm:main_politex}
If Assumption~\ref{asm:linear_realizability} holds, then for an arbitrarily small $\kappa>0$, by choosing $\tau = 1$, $\alpha=(1-\gamma)\sqrt{\frac{2\log(|\cA|)}{K}}$, $\lambda = \frac{\kappa^2(1-\gamma)^2}{256b^2}$, $K=\frac{32\log(|\cA|)}{\kappa^2(1-\gamma)^4}$, $n=\frac{1}{1-\gamma}\log(\frac{32\sqrt{d}(1+\log(1+\lambda^{-1}))}{(1-\gamma)^2\kappa})$, and $m=1024\frac{d (1+\log(1+\lambda^{-1}))}{\kappa^2(1-\gamma)^4}\log(\frac{8Kd(1+\log(1+\lambda^{-1}))}{\delta})$, we have with probability at least $1-\delta$, the mixture policy $\overline\pi_{K}$ that \textsc{Confident MC-Politex} outputs satisfies
\[
V^*(\rho) - V_{\overline{{\pi}}_K}(\rho) \le \kappa.
\]
Moreover, the query and computational costs for the algorithm are $\poly(d, \frac{1}{1-\gamma}, \frac{1}{\kappa}, \log(\frac{1}{\delta}), \log(b))$ and $\poly(d, \frac{1}{1-\gamma}, \frac{1}{\kappa}, |\cA|, \log(\frac{1}{\delta}), \log(b))$, respectively.

Alternatively, if Assumption~\ref{asm:linear_realizability_prox} holds, then by choosing $\tau=1$, $\alpha=(1-\gamma)\sqrt{\frac{2\log(|\cA|)}{K}}$, $\lambda = \frac{\epsilon^2 d}{b^2}$, $K=\frac{2\log(|\cA|)}{\epsilon^2 d (1-\gamma)^2}$, $n = \frac{1}{1-\gamma}\log(\frac{1}{\epsilon(1-\gamma)})$, and $m=\frac{1}{\epsilon^2(1-\gamma)^2}\log(
\frac{8Kd(1+\log(1+\lambda^{-1}))}{\delta})$, we have with probability at least $1-\delta$, the mixture policy $\overline\pi_K$ that \textsc{Confident MC-Politex} outputs satisfies
\[
V^*(\rho) - V_{\overline{\pi}_K}(\rho) \le \frac{42\epsilon\sqrt{d}}{1-\gamma}(1+\log(1+b^2\epsilon^{-2}b^{-1})).
\]
Moreover, the query and computational costs for the algorithm are $\poly(d, \frac{1}{1-\gamma}, \frac{1}{\epsilon}, \log(\frac{1}{\delta}), \log(b))$ and $\poly(d, \frac{1}{1-\gamma}, \frac{1}{\epsilon}, |\cA|, \log(\frac{1}{\delta}), \log(b))$, respectively.
\end{theorem}

We prove Theorem~\ref{thm:main_politex} in Appendix~\ref{apx:main_politex}. Here, we first discuss the query and computational costs of both algorithms and then provide a sketch of our proof.

\paragraph{Query and computational costs}
In our analysis, we say that we start a new \emph{loop} whenever we start (or restart) the policy iteration process, i.e., going to line $(*)$ in Algorithm~\ref{alg:main}.
By definition, when we start a new loop, the size of the core set $\cC$ is increased by $1$.
First, in Lemma~\ref{lem:coreset_size} below, we show that the size of the core set will never exceed $C_{\max}=\tilde{\cO}(d)$. Therefore, the total number of loops is at most $C_{\max}$. In each loop, we run $K$ policy iterations; in each iteration, we run Algorithm~\ref{alg:rollout} from at most $C_{\max}$ points from the core set; and each time when we run Algorithm~\ref{alg:rollout}, we query the simulator at most $\cO(mn)$ times. Thus, for both algorithms, the total number of queries that we make is at most $C_{\max}^2Kmn$. Therefore, using the parameter choice in Theorems~\ref{thm:main_rand} and~\ref{thm:main_politex} and omitting logarithmic factors, we can obtain the query costs of \textsc{Confident MC-LSPI} and \textsc{Politex} in Table \ref{table:comparsion_lspi_politex}. As we can see, when $\epsilon=0$ or $\epsilon\neq 0$ but $\epsilon=o(1/\sqrt{d})$ (the regime we care about in this paper), the query cost of \textsc{Confident MC-LSPI}~is lower than \textsc{Politex}.
As for computational cost, since our policy improvement steps only involve matrix multiplication and matrix inversion, the computational cost is also polynomial in the aforementioned factors. One thing to notice is that during the rollout process, in each step, the agent needs to compute the features of a state paired with all actions, and thus the computational cost linearly depends on $|\cA|$; on the contrary the query cost does not depend on $|\cA|$ since in each step the agent only needs to query the simulator with the action sampled according to the policy.


\paragraph*{Sub-optimality} We also note that when Assumption~\ref{asm:linear_realizability_prox} holds, i.e., $\epsilon\neq 0$, the sub-optimality of the output policy is $\tilde{\cO}(\frac{\epsilon\sqrt{d}}{(1-\gamma)^2})$ for \textsc{LSPI} and $\tilde{\cO}(\frac{\epsilon\sqrt{d}}{1-\gamma})$ for \textsc{Politex}. Therefore, in the presence of a model misspecification error, \textsc{Confident MC-Politex} can achieve a better final sub-optimality than \algoname, although it's query cost is higher.


\begin{table}[ht]
\centering
\vspace{-3mm}
\caption{Comparison of \textsc{Confident MC-LSPI} and \textsc{Politex}
}\label{table:comparsion_lspi_politex}
\begin{tabular}{ cccc } 
 \toprule
  & Query ($\epsilon=0$) & Query ($\epsilon \neq 0$) & Sub-optimality ($\epsilon \neq 0$) \\ 
 \midrule
 LSPI & $\tilde{\cO}\big(\frac{d^3}{\kappa^2(1-\gamma)^{8}}\big)$ & $\tilde{\cO}\big(\frac{d^2 }{\epsilon^2(1-\gamma)^4}\big)$ & $\tilde{\cO}\big(\frac{\epsilon\sqrt{d}}{(1-\gamma)^2}\big)$ \\ 
 \hline
 \textsc{Politex} & $\tilde{\cO}\big(\frac{d^3}{\kappa^4(1-\gamma)^9}\big)$ & $\tilde{\cO}\big(\frac{d}{\epsilon^4(1-\gamma)^5}\big)$ & $\tilde{\cO}\big(\frac{\epsilon\sqrt{d}}{1-\gamma}\big)$
 \\ \bottomrule
\end{tabular} 
\vspace{-3mm}
\end{table}

\paragraph*{Proof sketch} 
We now discuss our proof strategy, focusing on \textsc{LSPI} for simplicity. 
\paragraph{Step 1: Bound the size of the core set} The first step is to show that our algorithm will terminate. This is equivalent to showing that the size of the core set $\cC$ will not exceed certain finite quantity, since whenever we receive the $\uncertain$ status from \rolloutname, we increase the size of the core set by $1$, go back to line $(*)$ in Algorithm~\ref{alg:main}, and start a new loop. The following lemma shows that the size of the core set is always bounded, and thus the algorithm will always terminate.
\begin{lemma}\label{lem:coreset_size}
Under Assumption~\ref{asm:bounded_norm}, the size of the core set $\cC$ will not exceed
\[
C_{\max}:=\frac{e}{e-1}\frac{1+\tau}{\tau} d \left(\log(1+\frac{1}{\tau}) + \log(1+\frac{1}{\lambda})\right).
\]
\end{lemma}
This result first appears in~\citet{russo2013eluder} as the eluder dimension of linear function class. We present the proof of this lemma in Appendix~\ref{apx:coreset_size} for completeness.

\paragraph{Step 2: Virtual policy iteration}
The next step is to analyze the gap between the value of the optimal policy and the policy $\pi$ parameterized by the vector $w_{K-1}$ that the algorithm outputs in the final loop, i.e., $V^*(\rho) - V_{\pi_{K-1}}(\rho)$. For ease of exposition, here we only consider the case of deterministic probability transition kernel $P$. Our full proof in Appendix~\ref{apx:main_rand_v2} considers general stochastic dynamics.

To analyze our algorithm, we note that for approximate policy iteration (API) algorithms, if in every iteration (say the $k$-th iteration), we have an approximate Q-function that is close to the true Q-function of the rollout policy (say $\pi_{k-1}$) in $\ell_\infty$ norm, i.e., $\|Q_{k-1} - Q_{\pi_{k-1}}\|_\infty \le \eta$, then existing results~\citep{munos2003error,farahmand2010error} ensure that we can learn a good policy if in every iteration we choose the new policy to be greedy with respect to the approximate Q-function. However, since we only have local access to the simulator, we cannot have such $\ell_\infty$ guarantee. In fact, as we show in the proof, we can only ensure that when $\phi(s, a)$ is in the good set $\cH$, our linear function approximation is accurate, i.e., $|Q_{k-1}(s, a) - Q_{\pi_{k-1}}(s,a)|\le \eta$ where $Q_{k-1}(s, a) = w_k^\top \phi(s, a)$. To overcome the lack of $\ell_\infty$ guarantee, we introduce the notion of \emph{virtual policy iteration} algorithm. In the virtual algorithm, we start with the same initial policy $\tilde{\pi}_0 = \pi_0$. In the $k$-th iteration of the virtual algorithm, we assume that we have access to the true Q-function of the rollout policy $\tilde{\pi}_{k-1}$ when $\phi(s, a)\notin\cH$, and construct
\[
\tilde{Q}_{k-1}(s, a) = \begin{cases}
\tilde{w}_k^\top \phi(s, a) & ~\text{if}~\phi(s, a)\in\cH \\
Q_{\tilde\pi_{k-1}}(s,a) & ~\text{otherwise},
\end{cases}
\]
where $\tilde{w}_k$ is the linear coefficient that we learn in the virtual algorithm in the same way as in Eq.~\eqref{eq:ridge_sol}. 
Then $\tilde\pi_{k}$ is chosen to be greedy with respect to $\tilde{Q}_{k-1}(s, a)$.
In this way, we can ensure that $\tilde{Q}_{k-1}(s, a)$ is close to the true Q-function $Q_{\tilde\pi_{k-1}}(s,a)$ in $\ell_\infty$ norm and thus the output policy, say $\tilde{\pi}_{K-1}$, of the virtual algorithm is good in the sense that $V^*(\rho) - V_{\tilde\pi_{K-1}}(\rho)$ is small.

To connect the output policy of the virtual algorithm and our actual algorithm, we note that by definition, in the final loop of our algorithm, in any iteration, for any state $s$ that the agent visits in \rolloutname, and any action $a\in\cA$, we have that $\phi(s, a)\in\cH$ since the subroutine never returns $\uncertain$ status. Further, because the initial state, probability transition kernel, and the policies are all deterministic, we know that the rollout trajectories of the virtual algorithm and our actual algorithm are always the same in the final loop (the virtual algorithm does not get a chance to use the true Q-function $Q_{\tilde{\pi}_{k-1}}$). With rollout length $n$, we know that when we start with state $\rho$, the output of the virtual algorithm $\tilde\pi_{K-1}$ and our actual algorithm $\pi_{K-1}$ take exactly the same actions for $n$ steps, and thus $| V_{\pi_{K-1}}(\rho) - V_{\tilde\pi_{K-1}}(\rho) | \le \frac{\gamma^{n+1}}{1-\gamma}$, which implies that $V^*(\rho) - V_{\pi_{K-1}}(\rho)$ is small. To extend this argument to the setting with stochastic transitions, we need to use a coupling argument which we elaborate in the Appendix.
\section{Conclusion}

We propose the \algoname~and \textsc{Confident MC-Politex} algorithms, for local planning with linear function approximation. Under the assumption that the Q-functions of all policies are linear in some features of the state-action pairs, we show that our algorithm is query and computationally efficient. We introduce a novel analysis technique based on a virtual policy iteration algorithm, which can be used to leverage existing guarantees on approximate policy iteration with $\ell_\infty$-bounded evaluation error. We use this technique to show that our algorithm can learn the optimal policy for the given initial state even only with local access to the simulator. Future directions include extending our analysis technique to broader settings.

\section*{Acknowledgement}
The authors would like to thank Gell\'{e}rt Weisz for helpful comments.
\bibliographystyle{plainnat} 
{\small
\bibliography{ref}
}

\appendix
\section*{Appendix}
\section{Proof of Lemma~\ref{lem:coreset_size}}\label{apx:coreset_size}
This proof essentially follows the proof of the upper bound for the eluder dimension of a linear function class in~\citet{russo2013eluder}. We present the proof here for completeness.

We restate the core set construction process in the following way with slightly different notation. We begin with $\Phi_0 = 0$. In the $t$-th step, we have a core set with feature matrix $\Phi_{t-1}\in\R^{(t-1)\times d}$. Suppose that we can find $\phi_t\in\R^d$, $\|\phi_t\|_2 \le 1$, such that
\begin{align}\label{eq:new_feature_2}
\phi_t^\top (\Phi_{t-1}^\top \Phi_{t-1} + \lambda I)^{-1} \phi_t > \tau,
\end{align}
then we let $\Phi_t := [\Phi_{t-1}^\top~~ \phi_t]^\top\in\R^{t \times d}$, i.e., we add a row at the bottom of $\Phi_{t-1}$. If we cannot find such $\phi_t$, we terminate this process. We define $\Sigma_t := \Phi_t^\top\Phi_t + \lambda I$. It is easy to see that $\Sigma_0=\lambda I$ and $\Sigma_t = \Sigma_{t-1} + \phi_t\phi_t^\top$.

According to matrix determinant lemma~\citep{harville1998matrix}, we have
\begin{align}
    \det(\Sigma_t) &= (1+\phi_t^\top \Sigma_{t-1}^{-1}\phi_t)\det(\Sigma_{t-1}) > (1+\tau)\det(\Sigma_{t-1}) \nonumber \\
    & >\cdots > (1+\tau)^t \det(\Sigma_0) = (1+\tau)^t \lambda^d, \label{eq:mat_det}
\end{align}
where the inequality is due to~\eqref{eq:new_feature_2}. Since $\det(\Sigma_t)$ is the product of all the eigenvalues of $\Sigma_t$, according to the AM-GM inequality, we have
\begin{align}\label{eq:amgm}
    \det(\Sigma_t) \le \left(\frac{\tr(\Sigma_t)}{d}\right)^d = \left(\frac{\tr(\sum_{i=1}^t \phi_i\phi_i^\top) + \tr(\lambda I)}{d}\right)^d \le (\frac{t}{d} + \lambda)^d,
\end{align}
where in the second inequality we use the fact that $\|\phi_i\|_2\le 1$. Combining~\eqref{eq:mat_det} and~\eqref{eq:amgm}, we know that $t$ must satisfy
\[
(1+\tau)^t \lambda^d < (\frac{t}{d} + \lambda)^d,
\]
which is equivalent to
\begin{align}\label{eq:t_condition}
    (1+\tau)^{\frac{t}{d}} < \frac{t}{\lambda d} + 1.
\end{align}
We note that if $t \le d$, the result of the size of the core set in Lemma~\ref{lem:coreset_size} automatically holds. Thus, we only consider the situation here $t > d$. In this case, the condition~\eqref{eq:t_condition} implies
\begin{align}
    \frac{t}{d}\log(1+\tau) < \log(1+\frac{t}{\lambda d}) < \log(\frac{t}{d}(1+\frac{1}{\lambda})) &= \log(\frac{t}{d}) + \log(1+\frac{1}{\lambda}) \nonumber \\
    &= \log\left(\frac{t\tau}{d(1+\tau)}\right) + \log(\frac{1+\tau}{\tau}) + \log(1+\frac{1}{\lambda}). \label{eq:t_condition_2}
\end{align}
Using the fact that for any $x > 0$, $\log(1+x) > \frac{x}{1+x}$, and that for any $x>0$, $\log(x) \le \frac{x}{e}$, we obtain
\begin{align}\label{eq:t_condition_3}
  \frac{t\tau}{d(1+\tau)} <  \frac{t\tau}{ed(1+\tau)} +\log(\frac{1+\tau}{\tau}) + \log(1+\frac{1}{\lambda}),
\end{align}
which implies
\[
t < \frac{e}{e-1}\frac{1+\tau}{\tau} d \left(\log(1+\frac{1}{\tau}) + \log(1+\frac{1}{\lambda})\right).
\]

\section{Proof of Theorem~\ref{thm:main_rand}}\label{apx:main_rand_v2}

In this proof, we say that we start a new \emph{loop} whenever we start (or restart) the policy iteration process, i.e., going to line $(*)$ in Algorithm~\ref{alg:main}. In each loop, we have at most $K$ iterations of policy iteration steps. By definition, we also know that when we start a new loop, the size of the core set $\cC$ increases by $1$ compared with the previous loop. We first introduce the notion of virtual policy iteration algorithm. This virtual algorithm is designed to leverage the existing results on approximate policy iteration with $\ell_\infty$ bounded error in the approximate Q-functions~\citep{munos2003error,farahmand2010error}. We first present the details of the virtual algorithm, and then provide performance guarantees for the main algorithm.

\subsection{Virtual approximate policy iteration with coupling}\label{apx:virtual_api_coupling}

The virtual policy iteration algorithm is a virtual algorithm that we use for the purpose of proof. It is a version of approximate policy iteration (API) with a simulator. An important factor is that the simulators of the virtual algorithm and the main algorithm need to be \emph{coupled}, which we explain in this section. 

The virtual algorithm is defined as follows. Unlike the main algorithm, the virtual algorithm runs exactly $C_{\max}$ loops, where $C_{\max}$ is the upper bound for the size of the core set defined in Lemma~\ref{lem:coreset_size}. 
In the virtual algorithm, we let the initial policy be the same as the main algorithm, i.e., $\tilde\pi_0=\pi_0$. Unlike the main algorithm, the virtual algorithm runs exactly $K$ iterations of policy iteration.
In the $k$-th iteration ($k\ge 1$), the virtual algorithm runs rollouts from each element in the core set $\cC$ (we will discuss how the virtual algorithm constructs the core set later) with $\tilde\pi_{k-1}$ with a simulator where $\tilde\pi_{k-1}$ is in the form of Eq.~\eqref{def:virtual_policy} ($\tilde{Q}_{k-1}$ will be defined once we present the details of the virtual algorithm).

We now describe the rollout process of the virtual algorithm. We still use a subroutine similar to \rolloutname. The simulator of the virtual algorithm can still generate samples of next state given a state-action pair according to the probability transition kernel $P$. The major difference from the main algorithm is that during the rollout process, when we find a state-action pair whose feature is outside of the good set $\cH$ (defined in Definition~\ref{def:confident_set}), i.e., $(s, a)$ such that $\phi(s, a)^\top (\Phi_{\cC}^\top\Phi_{\cC}+\lambda I)\phi(s, a) > \tau$, we do not terminate the subroutine, instead we record this state-action pair along with its feature (we call it the \emph{recorded element}), and then keep running the rollout process using $\tilde{\pi}_{k-1}$. Two situations can occur at the end of each loop: 1) We did not record any element, in which case we use the same core set $\cC$ in the next loop, and 2) We have at least one recorded element in a particular loop, in which case we add the \emph{first} element to the core set and discard any other recorded elements. In other words, in each loop of the virtual algorithm, we find the first state-action pair (if any) whose feature is outside of the good set and add this pair to the core set. Another difference from the main algorithm is that in the virtual algorithm, we do not end the rollout subroutine when we identify an uncertain state-action pair, and as a result, the rollout subroutine in the virtual algorithm \emph{always returns} an estimation of the Q-function.


We now proceed to present the virtual policy iteration process. In the $k$-th iteration, the virtual algorithm runs $m$ trajectories of $n$-step rollout using the policy $\tilde{\pi}_{k-1}$ from each element $z\in\cC$, obtains the empirical average of the discounted return $z_q$ in the same way as in Algorithm~\ref{alg:rollout}. Then we concatenate them, obtain the vector $\tilde{q}_{\cC}$, and compute
\begin{align}\label{eq:linear_est_rand}
    \tilde{w}_k = (\Phi_{\cC}^\top \Phi_{\cC} + \lambda I)^{-1}\Phi_{\cC}^\top \tilde{q}_{\cC}.
\end{align}
We use the notion of good set $\cH$ defined in Definition~\ref{def:confident_set}, and define the \emph{virtual Q-function} as follows:
\begin{align}\label{eq:virtual_q_rand}
    \tilde{Q}_{k-1}(s, a) :=
    \begin{cases}
    \tilde{w}_{k}^\top \phi(s, a), & \phi(s, a) \in \cH, \\
    Q_{\tilde\pi_{k-1}}(s, a),  & \phi(s, a)  \notin \cH,
    \end{cases}
\end{align}
by assuming the access to the true Q-function $Q_{\tilde\pi_{k-1}}(s, a)$. The next policy $\tilde{\pi}_k$ is defined as the greedy policy with respect to $\tilde{Q}_{k-1}(s, a)$, i.e.,
\begin{equation}\label{def:virtual_policy}
 \tilde\pi_k(a | s) = \ind \left(a = \arg\max_{a'\in\cA} \tilde{Q}_{k-1}(s, a'))\right).    
\end{equation}
Recall that for the main algorithm, once we learn the parameter vector $w_k$, the next policy $\pi_k$ is greedy with respect to the linear function $w_k^\top\phi(s, a)$, i.e.,
\[
    \pi_k(a | s) = \ind \left(a = \arg\max_{a'\in\cA} w_k^\top\phi(s, a'))\right).
\]
For comparison, the key difference is that when we observe a feature vector $\phi(s, a)$ that is not in the good set $\cH$, our actual algorithm terminates the rollout and returns the state-action pair with the new direction, whereas the virtual algorithm uses the true Q-function of the state-action pair.

\paragraph*{Coupling} The major remaining issue now is how the main algorithm is connected to the virtual algorithm. We describe this connection with a coupling argument.
In a particular loop, for any positive integer $N$, when the virtual algorithm makes its $N$-th query in the $k$-th iteration to the virtual simulator with a state-action pair, say $(s_{\text{virtual}}, a_{\text{virtual}})$, if the main algorithm has not returned due to encountering an uncertain state-action pair, we assume that at the same time the main algorithm also makes its $N$-th query to the simulator, with a state-action pair, say $(s_{\text{main}}, a_{
\text{main}})$. We let the two simulators be \emph{coupled}: When they are queried with the same pair, i.e., $(s_{\text{main}}, a_{\text{main}}) = (s_{\text{virtual}}, a_{
\text{virtual}})$, the next states that they return are also the same. In other words, the simulator for the main algorithm samples $s_{\text{main}}'\sim P(\cdot | s_{\text{main}}, a_{\text{main}})$, and the virtual algorithm samples $s_{\text{virtual}}'\sim P(\cdot| s_{\text{virtual}}, a_{
\text{virtual}})$, and $s_{\text{main}}'$ and $s_{\text{virtual}}'$ satisfy the joint distribution such that $\Prob{s_{\text{main}}' = s_{\text{virtual}}'}=1$. In the cases where $(s_{\text{main}}, a_{\text{main}}) \neq (s_{\text{virtual}}, a_{
\text{virtual}})$ or the main algorithm has already returned due to the discovery of a new feature direction, the virtual algorithm samples from $P$ independently from the main algorithm. Note that this setup guarantees that both the virtual algorithm and the main algorithm have valid simulators which can sample from the same probability transition kernel $P$.

There are a few direct consequences of this coupling design. First, since the virtual and main algorithms start with the same initial core set elements (constructed using the initial state), we know that in any loop, when starting from the same core set element $z$, both algorithms will have \emph{exactly the same rollout trajectories} until the main algorithm identifies an uncertain state-action pair and returns. This is due to the coupling of the simulators and the fact that within the good set $\cH$, the policies for the main algorithm and the virtual algorithm take the same action. Later, we will discuss this point more in Lemma~\ref{lem:virtual_api_rand}. Second, the core set elements that the virtual and main algorithms use are exactly the same for any loop. This is because when the main algorithm identifies an uncertain state-action pair, it adds it to the core set and start a new loop, and the virtual algorithm also only adds the \emph{first recorded element} to the core set. Since the simulators are the coupled, the first uncertain state-action pair that they encounter will be the same, meaning that both algorithms always add the same element to the core set, until the main algorithm finishes its final loop. We note that the core set elements on our algorithm are stored as ordered list so the virtual and main algorithm always run rollouts with the same ordering of the core set elements. Another observation is that while the virtual algorithm has a deterministic number of loops $C_{\max}$, the total number of loops that the main algorithms may run is a random variable whose value cannot exceed $C_{\max}$.

The next steps of the proof are the following:
\begin{itemize}
    \item We show that in each loop, with high probability, the virtual algorithm proceeds as an approximate policy iteration algorithm with a bounded $\ell_\infty$ error in the approximate Q-function. Thus the virtual algorithm produces a good policy at the end of each loop. Then, since by Lemma~\ref{lem:coreset_size}, we have at most
    \begin{align}\label{eq:cmax_def}
    C_{\max}:=\frac{e}{e-1}\frac{1+\tau}{\tau} d \left(\log(1+\frac{1}{\tau}) + \log(1+\frac{1}{\lambda})\right)
    \end{align}
    loops, with a union bound, we know that with high probability, the virtual algorithm produces a good policy in all the loops.
    
    \item We show that due to the coupling argument, the output parameter vector in the main and the virtual algorithms, i.e., $w_{K-1}$ and $\tilde{w}_{K-1}$ in the final loop are the same. This leads to the conclusion that with the same initial state $\rho$, the value of the outputs of the main algorithm and the virtual algorithm are close, and thus the main algorithm also outputs a good policy.
\end{itemize}
We prove these two points in Sections~\ref{apx:virtual} and~\ref{apx:main}, respectively.

\subsection{Analysis of the virtual algorithm}\label{apx:virtual}

Throughout this section, we will consider a fixed loop of the virtual algorithm, say the $\ell$-th loop. We assume that at the beginning of this loop, the virtual algorithm has a core set $\cC_\ell$. Notice that $\cC_\ell$ is a random variable that only depends on the randomness of the first $\ell-1$ loops. In this section, we will first condition on the randomness of all the first $\ell-1$ loops and only consider the randomness of the $\ell$-th loop. Thus we will first treat $\cC_\ell$ as a deterministic quantity. For simplicity, we write $\cC:=\cC_\ell$.

Consider the $k$-th iteration of a particular loop of the virtual algorithm with core set $\cC$. We would like to bound $\|\tilde{Q}_{k-1} - Q_{\tilde\pi_{k-1}}\|_\infty$. First, we have the following lemma for the accuracy of the Q-function for any element in the core set. To simplify notation, in this lemma, we omit the subscript and use $\pi$ to denote a policy that we run rollout with in an arbitrary iteration of the virtual algorithm.
\begin{lemma}\label{lem:q_est_rand}
Let $\pi$ be a policy that we run rollout with in an iteration of the virtual algorithm.
Then, for any element $z\in\cC$ and any $\theta>0$, we have with probability at least $1-2\exp(-2\theta^2(1-\gamma)^2m)$,
\begin{align}\label{eq:est_err_true_q}
|z_q - Q_\pi(z_s, z_a)| \le \frac{\gamma^{n+1}}{1-\gamma} + \theta.
\end{align}
\end{lemma}
\begin{proof}
By the definition of $Q_\pi(z_s, z_a)$:
\[
Q_\pi(z_s, z_a) = \E_{s_{t+1}\sim P(\cdot | s_t, a_t), a_{t+1} \sim \pi(\cdot | s_{t+1})} \left[ \sum_{t=0}^\infty \gamma^t r(s_t, a_t) \mid s_0=z_s, a_0=z_a \right],
\]
and define the $n$-step truncated Q-function:
\[
Q_\pi^n(z_s, z_a) = \E_{s_{t+1}\sim P(\cdot | s_t, a_t), a_{t+1} \sim \pi(\cdot | s_{t+1})} \left[ \sum_{t=0}^n \gamma^t r(s_t, a_t) \mid s_0=z_s, a_0=z_a \right].
\]
Then we have $|Q_\pi^n(s, a) - Q_\pi(s, a)| \le \frac{\gamma^{n+1}}{1-\gamma}$.
Moreover, the Q-function estimate $z_q$ is an average of $m$ independent and unbiased estimates of $Q_\pi^n(s, a)$, which are all bounded in $[0, 1/(1-\gamma)]$. By Hoeffding's inequality we have with probability at least $1-2\exp(-2\theta^2(1-\gamma)^2m)$, $
|z_q - Q_\pi^n(s, a)| \le \theta$, which completes the proof.
\end{proof}

By a union bound over the $|\cC|$ elements in the core set, we know that 
\begin{align}\label{eq:est_error_all_c}
    \Prob{\forall~z\in\cC, |z_q - Q_{\tilde\pi_{k-1}}(z_s, z_a)| \le \frac{\gamma^{n+1}}{1-\gamma} + \theta} \ge 1-2C_{\max}\exp(-2\theta^2(1-\gamma)^2m).
\end{align}
The following lemma provides a bound on $|\tilde{Q}_{k-1}(s, a) - Q_{\tilde\pi_{k-1}}(s, a)|$, $\forall~(s, a)$ such that $\phi(s, a) \in \cH$.
\begin{lemma}\label{lem:est_error_rand}
Suppose that Assumption~\ref{asm:linear_realizability_prox} holds. Then, with probability at least 
\[
1-2C_{\max}\exp(-2\theta^2(1-\gamma)^2m),
\]
for any $(s, a)$ pair such that $\phi(s, a) \in \cH$, we have
\begin{align}\label{eq:q_gap_virtual_rand}
|\tilde{Q}_{k-1}(s, a) - Q_{\tilde\pi_{k-1}}(s, a)| \le b\sqrt{\lambda \tau} +  \big(\epsilon + \frac{\gamma^{n+1}}{1-\gamma} + \theta\big) \sqrt{\tau C_{\max}}+\epsilon:=\eta.
\end{align}
\end{lemma}

We prove this lemma in Appendix~\ref{apx:est_error_rand}. Since when $\phi(s, a)\notin \cH$, $\tilde{Q}_{k-1}(s, a) = Q_{\tilde{\pi}_{k-1}}(s, a)$, we know that $\|\tilde{Q}_{k-1}(s, a) - Q_{\tilde{\pi}_{k-1}}(s, a)\|_\infty \le \eta$. With another union bound over the $K$ iterations, we know that with probability at least
\[
1-2KC_{\max}\exp(-2\theta^2(1-\gamma)^2m),
\]
the virtual algorithm is an approximate policy iteration algorithm with $\ell_\infty$ bound $\eta$ for the approximation error on the Q-functions. We use the following results for API, which is a direct consequence of the results in~\citet{munos2003error,farahmand2010error}, and is also stated in~\citet{lattimore2020learning}. 
\begin{lemma}\label{lem:api_rand}
Suppose that we run $K$ approximate policy iterations and generate a sequence of policies $\pi_0, \pi_1, \ldots, \pi_K$. 
Suppose that for every $k=1,2,\ldots, K$, in the $k$-th iteration, we obtain a function $\tilde{Q}_{k-1}$ such that, $\|\tilde{Q}_{k-1} - Q_{\pi_{k-1}}\|_\infty\le \eta$, and choose $\pi_k$ to be greedy with respect to $\tilde{Q}_{k-1}$. Then
\[
\|Q^* - Q_{\pi_{K}}\|_\infty \le \frac{2\eta}{1-\gamma} + \frac{\gamma^K}{1-\gamma}.
\]
\end{lemma}
According to Lemma~\ref{lem:api_rand},
\begin{equation}\label{eq:tilde_q_result}
\|Q^* - Q_{\tilde\pi_{K-2}}\|_\infty \le  \frac{2\eta}{1-\gamma} + \frac{\gamma^{K-2}}{1-\gamma}.
\end{equation}
Then, since $\|Q_{\tilde\pi_{K-2}} - \tilde{Q}_{K-2}\|_\infty \le \eta$, we know that
\begin{align}\label{eq:tilde_q_result_2}
    \|Q^* - \tilde{Q}_{K-2}\|_\infty \le \frac{3\eta}{1-\gamma} + \frac{\gamma^{K-2}}{1-\gamma}.
\end{align}
The following lemma translates the gap in Q-functions to the gap in value.
\begin{lemma}\label{lem:api_value}
\citep{singh1994upper} Let $\pi$ be greedy with respect to a function $Q$. Then for any state $s$,
\[
V^*(s) - V_\pi(s) \le \frac{2}{1-\gamma}\|Q^*-Q\|_\infty.
\]
\end{lemma}
Since $\tilde{\pi}_{K-1}$ is greedy with respect to $\tilde{Q}_{K-2}$, we know that
\begin{align}\label{eq:gap_virtual_rand}
  V^*(\rho) - V_{\tilde{\pi}_{K-1}}(\rho) \le \frac{6\eta}{(1-\gamma)^2} + \frac{2\gamma^{K-2}}{(1-\gamma)^2}.  
\end{align}
We notice that this result is obtained by conditioning on all the previous $\ell-1$ loops and only consider the randomness of the $\ell$-th loop. More specifically, given any core set $\cC_\ell$ at the beginning of the $\ell$-th loop, we have
\[
\Prob{V^*(\rho) - V_{\tilde{\pi}_{K-1}}(\rho) \le \frac{6\eta}{(1-\gamma)^2} + \frac{2\gamma^{K-2}}{(1-\gamma)^2} \mid \cC_{\ell}} \ge 1-2KC_{\max}\exp(-2\theta^2(1-\gamma)^2m).
\]
By law of total probability we have
\begin{align*}
&\Prob{V^*(\rho) - V_{\tilde{\pi}_{K-1}}(\rho) \le \frac{6\eta}{(1-\gamma)^2} + \frac{2\gamma^{K-2}}{(1-\gamma)^2}} \\
=& \sum_{\cC_{\ell}} \Prob{V^*(\rho) - V_{\tilde{\pi}_{K-1}}(\rho) \le \frac{6\eta}{(1-\gamma)^2} + \frac{2\gamma^{K-2}}{(1-\gamma)^2} \mid \cC_{\ell}}\Prob{\cC_{\ell}} \\
\ge & 1-2KC_{\max}\exp(-2\theta^2(1-\gamma)^2m)\sum_{\cC_{\ell}}\Prob{\cC_{\ell}} \\
=& 1-2KC_{\max}\exp(-2\theta^2(1-\gamma)^2m).
\end{align*}
With another union bound over the $C_{\max}$ loops of the virtual algorithm, we know that with probability at least
\begin{align}\label{eq:err_prob_rand}
    1-2KC^2_{\max}\exp(-2\theta^2(1-\gamma)^2m),
\end{align}
Eq.~\eqref{eq:gap_virtual_rand} holds for all the loops. We call this event $\cE_1$ in the following.

\subsection{Analysis of the main algorithm}\label{apx:main}

We now move to the analysis of the main algorithm. Throughout this section, when we mention the \emph{final loop}, we mean the final loop of the \emph{main algorithm}, which may not be the final loop of the virtual algorithm. We have the following result.
\begin{lemma}\label{lem:virtual_api_rand}
In the final loop of the main algorithm, all the rollout trajectories in the virtual algorithm are exactly the same as those in the main algorithm, and therefore $w_k = \tilde{w}_k$ for all $1\le k\le K$.
\end{lemma}
\begin{proof}
We notice that since we only consider the final loop, in any iteration, for any state $s$ in all the rollout trajectories in the main algorithm, and all action $a\in\cA$, $\phi(s, a)\in\cH$. In the first iteration, since $\pi_0=\tilde\pi_0$, and the simulators are coupled, we know that all the rollout trajectories are the same between the main algorithm and the virtual algorithm, and as a result, all the Q-function estimates are the same, and thus $w_1 = \tilde{w}_1$. If we have $w_k=\tilde{w}_k$, we know that by the definition in~\eqref{eq:virtual_q_rand}, the policies $\pi_k$ and $\tilde{\pi}_k$ always take the same action given $s$ if for all $a\in\cA$, $\phi(s, a)\in\cH$. Again using the fact that the simulators are coupled, the rollout trajectories by $\pi_k$ and $\tilde\pi_k$ are also the same between the main algorithm and the virtual algorithm, and thus $w_{k+1}=\tilde{w}_{k+1}$.
\end{proof}

Since $\|\phi(s, a)\|_2\le 1$ for all $s, a$, we can verify that if we set $\tau \ge 1$, then after adding a state-action pair $s, a$ to the core set, then its feature vector $\phi(s, a)$ stays in the good set $\cH$. Recall that in the core set initialization stage of Algorithm~\ref{alg:main}, if for an action $a\in\cA$, $\phi(\rho, a)$ is not in $\cH$, we add $\rho, a$ to $\cC$. Thus, after the core set initialization stage, we have $\phi(\rho, a)\in\cH$ for all $a$. Thus $\pi_{K-1}(\rho) = \tilde\pi_{K-1}(\rho):=a_\rho$. Moreover, according to Lemma~\ref{lem:est_error_rand}, we also know that when $\cE_1$ happens,
\begin{align}\label{eq:v_to_w_virtual}
    | V_{\tilde\pi_{K-1}}(\rho) - \tilde{w}_{K}^\top\phi(\rho, a_\rho)| = |Q_{\tilde\pi_{K-1}}(\rho, a_\rho) - \tilde{w}_{K}^\top\phi(\rho, a_\rho)| \le \eta.
\end{align}
In the following, we bound the difference of the values of the output policy of the main algorithm $\pi_{K-1}$ and the output policy of the virtual algorithm $\tilde\pi_{K-1}$ in the final loop of the main algorithm, i.e., $|V_{\pi_{K-1}}(\rho) - V_{\tilde\pi_{K-1}}(\rho)|$. To do this, we use another auxiliary virtual policy iteration algorithm, which we call \emph{virtual-2} in the following. Virtual-2 is similar to the virtual policy iteration algorithm in Appendix~\ref{apx:virtual_api_coupling}.
The simulator of virtual-2 is coupled with the virtual algorithm, and virtual-2 also uses the same initial policy $\hat\pi_0:=\pi_0$ as the main algorithm. Virtual-2 also uses Monte-Carlo rollouts with the simulator and obtains the estimated Q-function values $\hat{q}_{\cC}$, and the linear regression coefficients are computed in the same way as~\eqref{eq:linear_est_rand}, i.e., $\hat{w}_k = (\Phi_{\cC}^\top \Phi_{\cC} + \lambda I)^{-1}\Phi_{\cC}^\top \hat{q}_{\cC}$. The virtual-2 algorithm also conducts uncertainty check in the rollout subroutine. Similar to the virtual algorithm, when it identifies an uncertain state-action pair, it records the pair and keeps running the rollout process. At the end of each loop, the virtual-2 algorithm still adds the first recorded element to the core set and discard other recorded elements.
The \emph{only difference} is that in virtual-2, we choose the virtual Q-function to be $\hat{Q}_{k-1}(s, a):=\hat{w}_{k}^\top\phi(s, a)$ for all $(s, a)\in\cS\times\cA$. Using the same arguments in Appendix~\ref{apx:virtual}, we know that with probability at least $1-2KC_{\max}^2\exp(-2\theta^2(1-\gamma)^2m)$, for all the loops and all the policy iteration steps in every loop, we have $|\hat{Q}_{k-1}(s, a) - Q_{\hat\pi_{k-1}}(s, a)|\le \eta$ for all $(s, a)$ such that $\phi(s, a)\in\cH$. We call this event $\cE_2$.
Since the simulators of virtual-2 is also coupled with that of the main algorithm, by the same argument as in Lemma~\ref{lem:virtual_api_rand}, we know that in the last iteration of the final loop of the main algorithm, we have $\hat{\pi}_{K-1} = \pi_{K-1}$ and $\hat{w}_{K} = w_{K}$.
We also know that when event $\cE_2$ happens, in the last iteration of the all the loops of virtual-2,
\begin{align}\label{eq:v_to_w_virtual_2}
   |V_{\hat{\pi}_{K-1}}(\rho) - \hat{w}_K^\top\phi(\rho, a_\rho)| \le \eta.
\end{align}
Therefore, when both events $\cE_1$ and $\cE_2$ happen, combining~\eqref{eq:v_to_w_virtual} and~\eqref{eq:v_to_w_virtual_2}, and using the fact that $\tilde{w}_{K} = w_K=\hat{w}_K$, we know that
\begin{equation*}
    \begin{split}
         &|V_{\pi_{K-1}}(\rho) - V_{\tilde\pi_{K-1}}(\rho)| = |V_{\hat{\pi}_{K-1}}(\rho) - V_{\tilde\pi_{K-1}}(\rho)|\\
         \leq & |V_{\hat{\pi}_{K-1}}(\rho) - \hat{w}_K^{\top}\phi(\rho, a_{\rho})|+|\hat{w}_K^{\top}\phi(\rho, a_{\rho})-\tilde{w}_K^{\top}\phi(\rho, a_{\rho})| +|\tilde{w}_K^{\top}\phi(\rho, a_{\rho})-V_{\tilde\pi_{K-1}}(\rho)|\\
         \leq &\eta+0+\eta = 2\eta.
    \end{split}
\end{equation*}
Combining this fact with~\eqref{eq:gap_virtual_rand} and using union bound, we know that with probability at least
\begin{align}\label{eq:err_prob_rand_2}
   1-4KC_{\max}^2\exp(-2\theta^2(1-\gamma)^2m),
\end{align}
with $C_{\max}$ defined as in~\eqref{eq:cmax_def}, we have
\begin{align}\label{eq:final_v}
    V^*(\rho) - V_{{\pi}_{K-1}}(\rho) \le \frac{8\eta}{(1-\gamma)^2} + \frac{2\gamma^{K-2}}{(1-\gamma)^2}.
\end{align}

Finally, we choose the appropriate parameters. Note that we would like to ensure that the success probability in Eq.~\eqref{eq:err_prob_rand_2} is at least $1-\delta$ and at the same time, the sub-optimality (right hand side of Eq.~\eqref{eq:final_v}) to be as small as possible. Suppose that Assumption~\ref{asm:linear_realizability} holds, i.e, $\epsilon=0$ in~\eqref{eq:q_gap_virtual_rand}. It can be verified that by choosing $\tau=1$, $\lambda = \frac{\kappa^2(1-\gamma)^4}{1024 b^2}$, $n=\frac{3}{1-\gamma}\log(\frac{4(1+\log(1+\lambda^{-1})d)}{\kappa(1-\gamma)})$, $\theta = \frac{\kappa(1-\gamma)^2}{64 \sqrt{d(1+\log(1+\lambda^{-1}))}}$, $K=2+\frac{2}{1-\gamma}\log(\frac{3}{\kappa(1-\gamma)})$, $m=4096 \frac{d(1+\log(1+\lambda^{-1}))}{\kappa^2(1-\gamma)^{6}}\log(\frac{8Kd(1+\log(1+\lambda^{-1}))}{\delta})$, we can ensure that the error probability is at most $1-\delta$ and $V^*(\rho) - V_{{\pi}_{K-1}}(\rho) \le \kappa$. Suppose that Assumption~\ref{asm:linear_realizability_prox} holds. It can be verified that by choosing $\tau=1$, $\lambda = \frac{\epsilon^2 d}{b^2}$, $n=\frac{1}{1-\gamma}\log(\frac{1}{\epsilon(1-\gamma)})$, $\theta=\epsilon$, $K=2 + \frac{1}{1-\gamma} \log\big(\frac{1}{\epsilon\sqrt{d}}\big)$, $m=\frac{1}{\epsilon^2(1-\gamma)^2}\log(\frac{8Kd(1+\log(1+\lambda^{-1}))}{ \delta})$, we can ensure that with probability at least $1-\delta$, 
\[
V^*(\rho) - V_{{\pi}_{K-1}}(\rho) \le \frac{74\epsilon\sqrt{d}}{(1-\gamma)^2}(1+\log(1+\lambda^{-1})).
\]

\section{Proof of Lemma~\ref{lem:est_error_rand}}\label{apx:est_error_rand}
To simplify notation, we write $\pi:=\tilde\pi_{k-1}$, $\tilde{Q}(\cdot, \cdot) := \tilde{Q}_{k-1}(\cdot, \cdot)$, and $\tilde{w} = \tilde{w}_{k}$ in this proof. According to Eq.~\eqref{eq:est_error_all_c}, with probability at least $1-2C_{\max}\exp(-2\theta^2(1-\gamma)^2m)$,
\[
 |z_q - Q_{\pi}(z_s, z_a)| \le \frac{\gamma^{n+1}}{1-\gamma} + \theta
\]
holds for all $z\in\cC$. We condition on this event in the following derivation. Suppose that Assumption~\ref{asm:linear_realizability_prox} holds. We know that there exists $w_{\pi}\in\R^d$ with $\|w_{\pi}\|_2\le b$ such that for any $s, a$, 
\[
|Q_{\pi}(s, a) - w_{\pi}^\top \phi(s, a)|\le \epsilon.
\]
Let $\xi:=\tilde{q}_{\cC} - \Phi_{\cC}w_{\pi}$. Then we have
\begin{align}\label{eq:bound_xi_prox}
    \|\xi\|_\infty \le \epsilon + \frac{\gamma^{n+1}}{1-\gamma} + \theta.
\end{align}

Suppose that for a state-action pair $s, a$, the feature vector $\phi:=\phi(s, a)\in\cH$, with $\cH$ defined in Definition~\ref{def:confident_set}. Then we have
\begin{align}
    |\tilde{Q}(s, a) - Q_{\pi}(s, a)| & \le |\phi^\top \tilde{w}  - \phi^\top w_{\pi}| + \epsilon  \nonumber \\
    &= |\phi^\top(\Phi_{\cC}^\top \Phi_{\cC} + \lambda I)^{-1}\Phi_{\cC}^\top (\Phi_{\cC}w_{\pi} + \xi) - \phi^\top w_{\pi}| + \epsilon \nonumber \\
    &\le \underbrace{| \phi^\top \big(I- (\Phi_{\cC}^\top \Phi_{\cC} + \lambda I)^{-1}\Phi_{\cC}^\top\Phi_{\cC}\big)w_{\pi} |}_{E_1} + \underbrace{|\phi^\top(\Phi_{\cC}^\top \Phi_{\cC} + \lambda I)^{-1}\Phi_{\cC}^\top \xi|}_{E_2} + \epsilon. \label{eq:split_rand}
\end{align}
We then bound $E_1$ and $E_2$ in~\eqref{eq:split_rand}. Similar to Appendix~\ref{apx:coreset_size}, let $\Phi_{\cC}^\top \Phi_{\cC} + \lambda I := V\Lambda V^\top$ be the eigendecomposition of $\Phi_{\cC}^\top \Phi_{\cC} + \lambda I$ with $\Lambda=\diag(\lambda_1,\ldots, \lambda_d)$ and $V$ being an orthonormal matrix. Notice that for all $i$, $\lambda_i \ge \lambda$. Let $\alpha = V^\top \phi$. Then for $E_1$, we have
\begin{align}
    E_1 &= |\phi^\top V \big(I - \Lambda^{-1}(\Lambda - \lambda I) \big) V^\top w_{\pi}| =\lambda |\phi^\top V \Lambda^{-1} V^\top w_{\pi}|  \nonumber \\
    & \le \lambda b \|\alpha^\top \Lambda^{-1}\|_2 = \lambda b\sqrt{\sum_{i=1}^d \frac{\alpha_i^2}{\lambda_i^2}} \nonumber \\
    & \le b\sqrt{\lambda}\sqrt{\sum_{i=1}^d \frac{\alpha_i^2}{\lambda_i}}, \label{eq:bound_e1_1}
\end{align}
where for the first inequality we use Cauchy-Schwarz inequality and the assumption that $\| w_{\pi_{k-1}} \|_2\le b$, and for the second inequality we use the fact that $\lambda_i \ge \lambda$. On the other hand, since we know that $\phi\in\cH$, we know that $\alpha^\top \Lambda^{-1} \alpha \le \tau$, i.e., $\sum_{i=1}^d \alpha_i^2 \lambda_i^{-1} \le \tau$. Combining this fact with~\eqref{eq:bound_e1_1}, we obtain
\begin{align}\label{eq:bound_e1_2}
    E_1 \le b\sqrt{\lambda\tau}.
\end{align}

We now bound $E_2$. According to H\"{o}lder's inequality, we have
\begin{align}
    E_2 &\le \| \phi^\top(\Phi_{\cC}^\top \Phi_{\cC} + \lambda I)^{-1}\Phi_{\cC}^\top \|_1 \|\xi\|_\infty  \nonumber \\
    & \le \| \phi^\top(\Phi_{\cC}^\top \Phi_{\cC} + \lambda I)^{-1}\Phi_{\cC}^\top \|_2  \|\xi\|_\infty \sqrt{|\cC|}  \nonumber \\
    &=\big( \phi^\top(\Phi_{\cC}^\top \Phi_{\cC} + \lambda I)^{-1}\Phi_{\cC}^\top \Phi_{\cC} (\Phi_{\cC}^\top \Phi_{\cC} + \lambda I)^{-1} \phi \big)^{1/2} \|\xi\|_\infty \sqrt{|\cC|} \nonumber \\
    &= \big( \alpha^\top \Lambda^{-1}(\Lambda - \lambda I)\Lambda^{-1}\alpha \big)^{1/2} \|\xi \|_\infty \sqrt{|\cC|} \nonumber \\
    &=\sqrt{\sum_{i=1}^d \alpha_i^2 \frac{\lambda_i-\lambda}{\lambda_i^2}} \|\xi \|_\infty \sqrt{|\cC|} \nonumber \\
    &\le  (\epsilon + \frac{\gamma^{n+1}}{1-\gamma} +\theta) \sqrt{\tau C_{\max}}, \label{eq:bound_e2}
\end{align}
where in the last inequality we use the facts that $\sum_{i=1}^d \alpha_i^2\lambda_i^{-1} \le \tau$, Eq.~\eqref{eq:bound_xi_prox}, and Lemma~\ref{lem:coreset_size}. We can then complete the proof by combining~\eqref{eq:bound_e1_2} and~\eqref{eq:bound_e2}.

\section{Proof of Theorem~\ref{thm:main_politex}}\label{apx:main_politex}

First, we state a general result in~\citet{szepesvari2021politex} on \textsc{Politex}. Notice that in this result, we consider an arbitrary sequence of approximate Q-functions $Q_k$, $k=0,\ldots, K-1$, which do not have to take the form of~\eqref{eq:q_prox_politex}.
\begin{lemma}[\citet{szepesvari2021politex}]\label{lem:politex}
Given an initial policy $\pi_0$ and a sequence of functions $Q_k:\cS\times\cA\mapsto[0, (1-\gamma)^{-1}]$, $k=0,\ldots, K-1$, construct a sequence of policies $\pi_1,\ldots, \pi_{K-1}$ according to~\eqref{eq:politex_policy} with $\alpha=(1-\gamma)\sqrt{\frac{2\log(|\cA|)}{K}}$, then, for any $s\in\cS$, the mixture policy $\overline{\pi}_K$ satisfies
\[
V^*(s) - V_{\overline{\pi}_K}(s) \le \frac{1}{(1-\gamma)^2}\sqrt{\frac{2\log(|\cA|)}{K}} + \frac{2\max_{0\le k\le K-1} \|Q_k - Q_{\pi_k}\|_\infty}{1-\gamma}.
\]
\end{lemma}
We then consider a virtual \textsc{Politex}~algorithm. Similar to the vanilla policy iteration algorithm, in the virtual \textsc{Politex}~algorithm, we begin with $\tilde{\pi}_0:=\pi_0$. In the $k$-th iteration, we run Monte Carlo rollout with policy $\tilde{\pi}_{k-1}$, and obtain the estimates of the Q-function values $\tilde{q}_{\cC}$. We then compute the weight vector
\[
\tilde{w}_k = (\Phi_{\cC}^\top \Phi_{\cC} + \lambda I)^{-1}\Phi_{\cC}^\top \tilde{q}_{\cC},
\]
and according to Lemma~\ref{lem:est_error_rand}, for any $\theta>0$, with probability at least $1-2C_{\max}\exp(-2\theta^2(1-\gamma)^2m)$, for all $(s, a)$ such that $\phi(s, a)\in\cH$,
\begin{align}\label{eq:q_gap_virtual_politex}
|\tilde{w}_k^\top \phi(s, a) - Q_{\tilde{\pi}_{k-1}}(s, a)| \le b\sqrt{\lambda \tau} +  \big(\epsilon + \frac{\gamma^{n+1}}{1-\gamma} + \theta\big) \sqrt{\tau C_{\max}}+\epsilon:=\eta.
\end{align}
Then we define the virtual Q-function as
\[
\tilde{Q}_{k-1}(s, a) :=
    \begin{cases}
    \Pi_{[0, (1-\gamma)^{-1}]}(\tilde{w}_{k}^\top \phi(s, a)), & \phi(s, a) \in \cH, \\
    Q_{\tilde\pi_{k-1}}(s, a),  & \phi(s, a)  \notin \cH,
    \end{cases}
\]
assuming we have access to the true Q-function $Q_{\tilde\pi_{k-1}}(s, a)$ when $\phi(s, a)  \notin \cH$. 
We let the policy of the $(k+1)$-th iteration $\tilde{\pi}_k$ be
\begin{align}\label{eq:politex_policy_virtual}
\tilde\pi_k(a|s) \propto \exp\left(\alpha \sum_{j=1}^{k-1} \tilde{Q}_{k-1}(s, a) \right).
\end{align}
Since we always have $Q_{\tilde\pi_{k-1}}(s, a)\in[0, (1-\gamma)^{-1}]$, the clipping at $0$ and $(1-\gamma)^{-1}$ can only improve the accuracy of the estimation of the Q-function. Therefore, we know that with probability at least $1-2C_{\max}\exp(-2\theta^2(1-\gamma)^2m)$, we have $\|\tilde{Q}_{k-1} - Q_{\tilde\pi_{k-1}}\|_\infty\le \eta$. Then, by taking a union bound over the $K$ iterations and using the result in Lemma~\ref{lem:politex}, we know that with probability at least $1-2KC_{\max}\exp(-2\theta^2(1-\gamma)^2m)$, for any $s\in\cS$, the virtual \textsc{Politex}~algorithm satisfies
\begin{align}\label{eq:politex_virtual}
V^*(s) - V_{\overline{\tilde{\pi}}_K}(s) \le \frac{1}{(1-\gamma)^2}\sqrt{\frac{2\log(|\cA|)}{K}} + \frac{2\eta}{1-\gamma},
\end{align}
where $\overline{\tilde{\pi}}_K$ is the mixture policy of $\tilde{\pi}_0,\ldots, \tilde{\pi}_{K-1}$. Using another union bound over the $C_{\max}$ loops, we know that with probability at least $1-2KC_{\max}^2\exp(-2\theta^2(1-\gamma)^2m)$,~\eqref{eq:politex_virtual} holds for all the loops. We call this event $\cE_1$ in the following.

We then consider the virtual-2 \textsc{Politex}~algorithm. Similar to LSPI, the virtual-2 algorithm begins with $\hat{\pi}_0:=\pi_0$. In the $k$-th iteration, we run Monte Carlo rollout with policy $\hat{\pi}_{k-1}$, and obtain the estimates of the Q-function values $\hat{q}_{\cC}$. We then compute the weight vector
\[
\hat{w}_k = (\Phi_{\cC}^\top \Phi_{\cC} + \lambda I)^{-1}\Phi_{\cC}^\top \hat{q}_{\cC},
\]
and according to Lemma~\ref{lem:est_error_rand}, for any $\theta>0$, with probability at least $1-2C_{\max}\exp(-2\theta^2(1-\gamma)^2m)$, for all $(s, a)$ such that $\phi(s, a)\in\cH$,
\begin{align}\label{eq:q_gap_virtual2_politex}
|\hat{w}_k^\top \phi(s, a) - Q_{\hat{\pi}_{k-1}}(s, a)| \le \eta,
\end{align}
where $\eta$ is defined as in~\eqref{eq:q_gap_virtual_politex}. We also note that in the rollout process of the virtual-2 algorithm, we do not conduct the uncertainty check, i.e., we do not check whether the features are in the good set $\cH$. By union bound, we know that with probability at least $1-2KC_{\max}^2\exp(-2\theta^2(1-\gamma)^2m)$,~\eqref{eq:q_gap_virtual2_politex} holds for all the $K$ iterations of all the $C_{\max}$ loops. We call this event $\cE_2$ in the following. In the virtual-2 algorithm, we define the approximate Q-function in the same way as the main algorithm, i.e., we define
\[
\hat{Q}_{k-1}(s, a):=\Pi_{[0, (1-\gamma)^{-1}]}(\hat{w}_k^\top\phi(s, a)),
\]
and we let the policy of the $(k+1)$-th iteration be
\begin{align}\label{eq:politex_policy_virtual_2}
\hat\pi_k(a|s) \propto \exp\left(\alpha \sum_{j=1}^{k-1} \hat{Q}_{k-1}(s, a) \right).
\end{align}

We still let the simulators of all the algorithms be \emph{coupled} in the same way described as in Appendix~\ref{apx:virtual_api_coupling}. In addition, we also let the agent in the main algorithm be \emph{coupled} with the virtual and virtual-2 algorithm. Take the main algorithm and the virtual algorithm as an example. Recall that in the $k$-th iteration of a particular loop, the main algorithm and the virtual algorithm use rollout policies $\pi_{k-1}$ and $\tilde{\pi}_{k-1}$, respectively. In the \rolloutname~subroutine, the agent needs to sample actions according to the policies given a state. Suppose that in the $N$-th time that the agent needs to take an action, the main algorithm is at state $s_{\text{main}}$ and the virtual algorithm is at state $s_{\text{virtual}}$. If the two states are the same, i.e., $s_{\text{main}} = s_{\text{virtual}}$ and two distributions of actions given this state are also the same, i.e., $\pi_{k-1}(\cdot | s_{\text{main}}) = \tilde{\pi}_{k-1}(\cdot |  s_{\text{virtual}})$, then the actions that the agent samples in the main algorithm and the virtual algorithm are also the same. This means that the main algorithm samples $a_{\text{main}}\sim \pi_{k-1}(\cdot | s_{\text{main}})$ and the virtual algorithm samples $a_{\text{virtual}}\sim \pi_{k-1}(\cdot | s_{\text{virtual}})$, and with probability $1$, $a_{\text{main}}=a_{\text{virtual}}$. Otherwise, when $s_{\text{main}} \neq s_{\text{virtual}}$ or $\pi_{k-1}(\cdot | s_{\text{main}}) \neq \tilde{\pi}_{k-1}(\cdot |  s_{\text{virtual}})$, the main algorithm and the virtual algorithm samples a new action independently. The main algorithm and the virtual-2 algorithm are coupled in the same way. We note that using the same argument as in Lemma~\ref{lem:virtual_api_rand}, for the \emph{final loop} of the main algorithm, all the rollout trajectories of the main, virtual, and virtual-2 algorithms are the same, which implies that $w_k=\tilde{w}_k=\hat{w}_k$ for all $1\le k\le K$. This also implies that in the final loop of the main algorithm, all the policies in the $K$ iterations are the same between the main and the virtual-2 algorithm, i.e., $\pi_k=\hat{\pi}_k$, $0\le k \le K-1$. Moreover, for any state $s$ such that $\phi(s, a)\in\cH$ for all $a\in\cA$, we have $\pi_k(\cdot | s)=\tilde{\pi}(\cdot|s) = \hat\pi_k(\cdot | s)$. Since the initial state $\rho$ satisfies the condition that $\phi(\rho, a)\in\cH$ for all $a\in\cA$, we have $\pi_k(\cdot | \rho)=\tilde{\pi}(\cdot|\rho) = \hat\pi_k(\cdot | \rho)$.

Let $\overline{\hat{\pi}}_K$ be the policy that is uniformly chosen from $\hat{\pi}_0, \ldots, \hat{\pi}_{K-1}$ in the virtual-2 algorithm in the final loop of the main algorithm, and $\overline{\pi}_K$ be the policy that is uniformly chosen from $\pi_0, \ldots, \pi_{K-1}$ in the final loop of the main algorithm. Then we have
\begin{align}\label{eq:politex_virtual_2_main}
    |V_{\overline{\hat{\pi}}_K}(\rho) - V_{\overline\pi_K}(\rho)| = \left|\frac{1}{K}\sum_{k=0}^{K-1} (V_{\hat\pi_k}(\rho) - V_{\pi_k}(\rho)) \right| = 0,
\end{align}
and when events $\cE_1$ and $\cE_2$ happen,
\begin{align}
    &|V_{\overline{\hat{\pi}}_K}(\rho) - V_{\overline{\tilde\pi}_K}(\rho)| 
    \nonumber \\
    \le& \frac{1}{K}\sum_{k=0}^{K-1}|V_{\hat\pi_k}(\rho) - V_{\tilde\pi_k}(\rho)| \nonumber  \\
    =&\frac{1}{K}\sum_{k=0}^{K-1}\left|\sum_{a\in\cA}(\hat{\pi}_k(a|\rho)Q_{\hat\pi_k}(\rho, a) - \tilde{\pi}_k(a|\rho)Q_{\tilde\pi_k}(\rho, a))\right| \nonumber \\
    \le& \frac{1}{K}\sum_{k=0}^{K-1}\sum_{a\in\cA}\pi_k(a|\rho) \left|Q_{\hat\pi_k}(\rho, a) - Q_{\tilde\pi_k}(\rho, a))\right| \nonumber \\
    \le& \frac{1}{K}\sum_{k=0}^{K-1}\sum_{a\in\cA}\pi_k(a|\rho)\left| Q_{\hat\pi_k}(\rho, a) - \hat{w}_k^\top\phi(\rho, a) +\hat{w}_k^\top\phi(\rho, a) - \tilde{w}_k^\top\phi(\rho, a) + \tilde{w}_k^\top\phi(\rho, a) -  Q_{\tilde\pi_k}(\rho, a)) \right| \nonumber \\
    \le & \frac{1}{K}\sum_{k=0}^{K-1}\sum_{a\in\cA}\pi_k(a|\rho)\left(|Q_{\hat\pi_k}(\rho, a) - \hat{w}_k^\top\phi(\rho, a)| + |\hat{w}_k^\top\phi(\rho, a) - \tilde{w}_k^\top\phi(\rho, a)| + |\tilde{w}_k^\top\phi(\rho, a) -  Q_{\tilde\pi_k}(\rho, a))|\right) \nonumber \\
    \le &\frac{1}{K}\sum_{k=0}^{K-1}\sum_{a\in\cA}\pi_k(a|\rho) (\eta + 0 +\eta) = 2\eta. \label{eq:politex_virtual_2_virtual}
\end{align}
By combining~\eqref{eq:politex_virtual},~\eqref{eq:politex_virtual_2_main}, and~\eqref{eq:politex_virtual_2_virtual}, and using a union bound, we obtain that with probability at least $1-4KC_{\max}^2\exp(-2\theta^2(1-\gamma)^2m)$,
\begin{align}\label{eq:politex_main}
   V^*(\rho) - V_{\overline{{\pi}}_K}(\rho) \le \frac{1}{(1-\gamma)^2}\sqrt{\frac{2\log(|\cA|)}{K}} + \frac{4\eta}{1-\gamma}.
\end{align}
Now we choose appropriate parameters to obtain the final result. When Assumption~\ref{asm:linear_realizability} holds, i.e., $\epsilon=0$, one can verify that when we choose $\tau = 1$, $\lambda = \frac{(1-\gamma)^2\kappa^2}{256b^2}$, $K=\frac{32\log(|\cA|)}{\kappa^2(1-\gamma)^4}$, $n=\frac{1}{1-\gamma}\log(\frac{32\sqrt{d}(1+\log(1+\lambda^{-1}))}{(1-\gamma)^2\kappa})$, $\theta = \frac{(1-\gamma)\kappa}{32\sqrt{d(1+\log(1+\lambda^{-1}))}}$, and $m=\frac{1024 d (1+\log(1+\lambda^{-1}))}{(1-\gamma)^4\kappa^2}\log(\frac{8Kd(1+\log(1+\lambda^{-1}))}{\delta})$, we can ensure that with probability at least $1-\delta$, $V^*(\rho) - V_{\overline{{\pi}}_K}(s) \le \kappa$. When Assumption~\ref{asm:linear_realizability_prox} holds, one can verify that when we choose $\tau=1$, $\lambda = \frac{\epsilon^2 d}{b^2}$, $K=\frac{2\log(|\cA|)}{\epsilon^2 d (1-\gamma)^2}$, $\theta=\epsilon$, $n = \frac{1}{1-\gamma}\log(\frac{1}{\epsilon(1-\gamma)})$, and $m=\frac{1}{\epsilon^2(1-\gamma)^2}\log(
\frac{8Kd(1+\log(1+\lambda^{-1}))}{\delta})$, we can ensure that with probability at least $1-\delta$,
\[
V^*(\rho) - V_{\overline{\pi}_K}(\rho) \le \frac{42\epsilon\sqrt{d}}{1-\gamma}(1+\log(1+\lambda^{-1})).
\]

\section{Random initial state}\label{apx:random_initial}

We have shown that with a deterministic initial state $\rho$, our algorithm can learn a good policy. In fact, if the initial state is random, and the agent is allowed to sample from a distribution of the initial state, denoted by $\rho$ in this section, then we can use a simple reduction to show that our algorithm can still learn a good policy. In this case, the optimality gap is defined as the difference between the expected value of the optimal policy and the learned policy, where the expectation is taken over the initial state distribution, i.e., we hope to guarantee that $\E_{s\sim\rho}[V^*(s) - V_{\pi}(s)]$ is small.

The reduction argument works as follows. First, we add an auxiliary state $s_\init$ to the state space $\cS$ and assume that the algorithm starts from $s_\init$. From $s_\init$ and any action $a\in\cA$, we let the distribution of the next state be $\rho\in\Delta_{\cS}$, i.e., $P(\cdot | s_\init, a)=\rho$. We also let $r(s_\init, a)=0$. Then, for any policy $\pi$, we have $\E_{s\sim\rho}[V_\pi(s)] = \frac{1}{\gamma}V_\pi(s_\init)$. As for the features, for any $(s, a)\in\cS\times\cA$, we add an extra $0$ as the last dimension of the feature vector, i.e., we use $\phi^+(s, a)=[\phi(s, a)^\top~0]^\top\in\R^{d+1}$. For any $a\in\cA$, we let $\phi^+(s_\init, a)=[0\cdots 0~1]^\top\in\R^{d+1}$.
Note that this does not affect linear realizability except a change in the upper bound on the $\ell_2$ norm of the linear coefficients. Suppose that Asumption~\ref{asm:linear_realizability} holds. Suppose that in the original MDP, we have $Q_{\pi}(s, a)=w_\pi^\top\phi(s, a)$ with $w_{\pi}\in\R^d$. Let us define $w_{\pi}^+ = [w_\pi^\top~V_{\pi}(s_\init)]^\top\in\R^{d+1}$. Then, for any $s\neq s_\init$, we still have $Q_{\pi}(s, a)=(w_\pi^+)^\top\phi^+(s, a)$ since the last coordinate of $\phi^+(s, a)$ is zero. For $s_\init$, we have $Q_\pi(s_\init, a) = V_{\pi}(s_\init) = (w_\pi^+)^\top\phi^+(s, a)$. The only difference is that we now have $\|w^+_\pi\|_2 \le \sqrt{b^2 + (\frac{\gamma}{1-\gamma})^2}$ since we always have $0\le V_{\pi}(s_\init) \le \frac{\gamma}{1-\gamma}$.

Then the problem reduces to the deterministic initial state case with initial state $s_\init$. In the first step of the algorithm, we let $\cC=\{(s_\init, a, \phi^+(s_\init, a), \none)\}$. During the algorithm, to run rollout from any core set element $z$ with $z_s \in \cS$, we can use the current version of Algorithm~\ref{alg:rollout}. To run rollout from $(s_\init, a)$, we simply sample from $\rho$ as the first ``next state'' and then use the simulator to keep running the following trajectory of the rollout process.

\end{document}